\documentclass{article}
\usepackage{arxiv}
\usepackage{lineno,hyperref}
\modulolinenumbers[5]

\usepackage{natbib}

\usepackage{amssymb}

\usepackage{times}
\usepackage{soul}
\usepackage{xcolor}
\usepackage{url}
\usepackage{multirow, rotating}
\usepackage[utf8]{inputenc}
\usepackage[small]{caption}
\usepackage{graphicx}
\usepackage{amsmath}
\usepackage{amsthm, amssymb}
\usepackage{mathtools}
\usepackage{booktabs}
\usepackage[vlined,linesnumbered,ruled,titlenotnumbered]
{algorithm2e}
\usepackage{xspace}
\usepackage{dsfont}
\newcommand{\ignore}[1]{}
\usepackage{graphicx}
\newtheorem{definition}{Definition}
\newtheorem{lemma}{Lemma}
\newtheorem{theorem}{Theorem}
\newtheorem{corollary}{Corollary}
\newcommand{\ea}{(1+1)~EA\xspace}
\newcommand{\gsemo}{GSEMO\xspace}

\title{Runtime Analysis of Single- and Multi-Objective Evolutionary Algorithms for Chance Constrained Optimization Problems with Normally Distributed Random Variables}
\author{Frank Neumann\\
Optimisation and Logistics\\
School of Computer Science\\
The University of Adelaide\\
Adelaide, Australia
\And
Carsten Witt\\
Algorithms, Logic and Graphs\\
DTU Compute\\ Technical University of Denmark\\
2800 Kgs. Lyngby Denmark
}

\begin{document}

\maketitle
\begin{abstract}
  Chance constrained optimization problems allow to model problems where constraints involving stochastic components should only be violated with a small probability. Evolutionary algorithms have been applied to this scenario and shown to achieve high quality results. With this paper, we contribute to the theoretical understanding of evolutionary algorithms for chance constrained optimization. We study the scenario of stochastic components that are independent and normally distributed.
  Considering the simple single-objective (1+1)~EA, we show that imposing an additional uniform constraint already leads to local optima for very restricted scenarios and an exponential optimization time. We therefore introduce a multi-objective formulation of the problem which trades off the expected cost and its variance. We show that multi-objective evolutionary algorithms are highly effective when using this formulation and obtain a set of solutions that contains an optimal solution for any possible confidence level imposed on the constraint. Furthermore, we prove that this approach can also be used to compute a set of optimal solutions for the chance constrained minimum spanning tree problem. In order to deal with potentially exponentially many trade-offs in the multi-objective formulation, we propose and analyze improved convex multi-objective approaches.  Experimental investigations on instances of the NP-hard stochastic minimum weight dominating set problem  confirm the benefit of the multi-objective and the improved convex multi-objective approach in practice. 
\end{abstract}

\section{Introduction}
Many real-world optimization problems involve solving optimization problems that contain stochastic components~\citep{BEN:09}.
Chance constraints~\citep{Charnes} allow to limit the probability of violating a constraint involving stochastic components. In contrast to limiting themselves to ruling out constraint violations completely, this allows to deal with crucial constraints in a way that allows to ensure meeting the constraints with high confidence (usually determined by a confidence level $\alpha$)
while still maintaining solutions of high quality with respect to the given objective function.

Evolutionary algorithms have successfully been applied to chance constrained optimization problems~\citep{poojari,Zhang}. Recent studies investigated the classical knapsack problem in static \citep{DBLP:conf/gecco/XieHAN019,DBLP:conf/gecco/XieN020,DBLP:journals/corr/abs-2204-05597} and dynamic settings~\citep{DBLP:conf/ecai/AssimiHXN020} as well as complex stockpile blending problems~\citep{DBLP:conf/gecco/XieN021} and the optimization of submodular functions~\citep{DBLP:conf/ppsn/NeumannN20}.
Theoretical analyses for submodular problems with chance constraints, where each stochastic component is uniformly distributed and has the same amount of uncertainty, have shown that greedy algorithms and evolutionary Pareto optimization approaches only lose a small amount in terms of approximation quality when comparing against the corresponding deterministic problems~\citep{DBLP:conf/aaai/DoerrD0NS20,DBLP:conf/ppsn/NeumannN20} and that evolutionary algorithms significantly outperform the greedy approaches in practice.
Other recent theoretical runtime analyses of evolutionary algorithms have produced initial results for restricted classes of instances of the knapsack problem where the weights are chosen randomly~\citep{DBLP:conf/foga/0001S19,DBLP:conf/gecco/XieN0S21}.

For our theoretical investigations, we use runtime analysis which is a major theoretical tool for analyzing evolutionary algorithms in discrete search spaces~\citep{NeumannW10,Jansen13,DoerrN20}.
In order to understand the working behaviour of evolutionary algorithms on broader classes of problems with chance constraints, we consider the optimization of linear functions with respect to chance constraints where the stochastic components are independent and each weight $w_i$ is chosen according  to a normal distribution $N(\mu_i, \sigma_i^2)$. This allows to reformulate the problem by a deterministic equivalent non-linear formulation involving a linear combination of the expected value and the standard deviation of a given solution. 

We investigate how evolutionary algorithms can deal with chance constrained problems where the stochastic elements follow a normal distribution.
We first analyze the classical (1+1)~EA on the single-objective formulation given by the linear combination of the expected expected value and the standard deviation. We show that imposing a simple cardinality constraint for a simplified class of instances leads to local optima and exponential lower bounds for the (1+1)~EA.
In order to deal with the issue of the \ea not being able to handle even simple constraints due to the non-linearity of the objective functions, we introduce a Pareto optimization approach for the chance constrained optimization problems under investigation.
So far, Pareto optimization approaches that achieved provably good solution provided a trade-off with respect to the original objective functions and given constraints. In contrast to this, our approach trades off the different components determining the uncertainty of solutions, namely the expected value and variance of a solution.
A crucial property of our reformulation is that the extreme points of the Pareto front provide optimal solutions for any linear combination of the expected value and the standard deviation and solves the original chance constrained problem for any confidence level~$\alpha \geq 1/2$. These insights mean that the users of the evolutionary multi-objective algorithm does not need to know the desired confidence level in advance, but can pick from a set of trade-offs with respect to the expected value and variance for all possible confidence levels. We show that this approach can also be applied to the chance constrained minimum spanning tree problem examined in~\cite{DBLP:journals/dam/IshiiSNN81} where each edge cost is chosen independently according to its own normal distribution. In terms of algorithms, we analyze the well-known \gsemo~\citep{1299908} which has been frequently applied in the context of Pareto optimization~\citep{DBLP:journals/nc/NeumannW06,DBLP:journals/algorithmica/KratschN13,DBLP:journals/ec/FriedrichN15,DBLP:books/sp/ZhouYQ19} and show that it computes such an optimal set of solutions for any confidence level of $\alpha \geq 1/2$ in expected polynomial time if the population size stays polynomial with respect to the given inputs. 

In order to deal with potentially exponentially many trade-offs in the objective space, which may result in an exponential population size for \gsemo, we introduce convex hull based evolutionary multi-objective algorithms in Section~\ref{sec:convex-MOEA}. We show that they compute for every possible confidence level of $\alpha \geq 1/2$ an optimal solution in expected polynomial time even if the number of trade-offs with respect to the two objectives becomes exponential. The developed approaches include a convex version of \gsemo called Convex \gsemo which allows to limit the population size to $n^2$ (where $n$ is the problem size) while still obtaining optimal solutions for the chance constrained problem with uniform constraints as well as the chance constrained minimum spanning tree problem. As a result, Convex \gsemo obtains a 2-approximation also for the classical multi-objective minimum spanning tree problem which improves upon the result given in \cite{DBLP:journals/eor/Neumann07} where a pseudo-polynomial runtime has been proved for \gsemo to achieve a $2$-approximation.

Finally, we carry out extensive experimental investigations for different graphs and stochastic settings of the NP-hard minimum weight dominating set problem in Section~\ref{sec:experiments}.
We study cases where expected cost values and variances are chosen independently and uniformly at random, degree-based expected costs as well as settings where expected cost values and variances are negatively correlated. 
We experimentally compare the (1+1)~EA, \gsemo and Convex \gsemo for these settings. Our results show that the multi-objective approaches outperform the \ea in most cases. Furthermore, we point out that Convex \gsemo clearly outperforms \gsemo for the negatively correlated setting which can be attributed to the much smaller population size of the algorithm.

This article extends its conference version~\citep{DBLP:conf/ijcai/0001W22} in Section~\ref{sec:convex-MOEA} by introducing and analyzing convex evolutionary multi-objective algorithms that allow to work with a polynomial population size while still achieving the same theoretical performance guarantees as the standard \gsemo algorithm. Furthermore, this article includes more comprehensive experimental investigations of the considered algorithms than the conference version in Section~\ref{sec:experiments}.

The article is structured as follows. In Section~\ref{sec2}, we introduce the chanced constrained problems that we investigate in this article.  We provide exponential lower bounds on the optimization time of \ea in Section~\ref{sec3}. Our multi-objective approach is introduced and analyzed in Section~\ref{sec:MOEA} and we provide a convex evolutionary multi-objective algorithm that can deal with potentially exponentially many trade-offs in expected polynomial time in Section~\ref{sec:convex-MOEA}. Finally, we report on our experimental results in Section~\ref{sec:experiments} and finish with some concluding remarks.

\section{Chance Constrained Optimization Problems}
\label{sec2}

Our basic chance-constrained setting is given as follows. Given a set of $n$ items $E=\{e_1, \ldots, e_n\}$ with weights $w_i$, $1 \leq i \leq n$, we want to solve 
\begin{equation}
\min W  \text{~~~~subject to~~~~}  \mathrm{Pr}( w(x) \leq W) \geq \alpha,
\label{chance-problem}
\end{equation}
where $w(x) = \sum_{i=1}^n w_i x_i$, $x \in \{0,1\}^n$, and $\alpha \in \mathopen{[}0,1\mathclose{]}$.
Throughout this paper, we assume that the weights are independent and each $w_i$ is distributed according to a normal distribution $N(\mu_i, \sigma_i^2)$, $1 \leq i \leq n$, where $\mu_i \geq 1$ and $\sigma_i\geq 1$, $1 \leq i \leq n$. We denote by $\mu_{\max} = \max_{1 \leq i \leq n}  \mu_i$ and $v_{\max} = \max_{1 \leq i \leq n} \sigma_i^2$ the maximal expected value and maximal variance, respectively.
According to \cite{DBLP:journals/dam/IshiiSNN81}, the problem given in Equation~\ref{chance-problem} is in this case equivalent to minimizing
\begin{equation}
    g(x) = \sum_{i=1}^n \mu_i x_i + K_{\alpha} \cdot \left(\sum_{i=1}^n \sigma_i^2 x_i \right)^{1/2},
    \label{eq:sumofidentityandsquare}
\end{equation}
where $K_{\alpha}$ is the $\alpha$-fractile point of the standard normal distribution. Throughout this paper, we assume $\alpha \in \mathopen{[}1/2,1\mathclose{[}$ as we are interested in solutions of high confidence. Note that there is no finite value of $K_{\alpha}$ for $\alpha=1$ due to the infinite tail of the normal distribution. Our range of $\alpha$ implies $K_{\alpha}\geq 0$.

We carry out our investigations where there are additional constraints. 
First, we consider the additional constraint $|x|_1 \geq k$, which requires that at least $k$ items are chosen in each feasible solution. 
Furthermore, we consider the formulation of the stochastic minimum spanning tree problem given in \cite{DBLP:journals/dam/IshiiSNN81}. Given an undirected connected weighted graph $G=(V,E)$, $n=|V|$ and $m=|E|$ with random weights $w(e_i)$, $e_i \in E$. The search space is $\{0,1\}^m$. For a search point $x \in \{0,1\}^m$, we have
$w(x) = \sum_{i=1}^m w(e_i) x_i$ as the weight of a solution $x$.  
We investigate the problem given in Equation~\ref{chance-problem} and require for a solution $x$ to be feasible that $x$ encodes a connected graph. We do not require a solution to be a spanning tree in order to be feasible as removing an edge from a cycle in a connected graph automatically improves the solution quality and is being taken care of by the multi-objective algorithms we analyze in this paper. 
Note that the only difference compared to the previous setting involving the uniform constraint is the requirement that a feasible solution has to be a connected graph.

\section{Analysis of \ea}
\label{sec3}
The \ea (Algorithm~\ref{alg:EA}) is a simple evolutionary algorithm using independent bit flips and elitist selection. It is very well studied in the theory of evolutionary computation \citep{DoerrProbabilisticTools} and serves as a stepping stone towards the analysis of more complicated evolutionary algorithms. As common  in the area of runtime analysis, we measure the runtime of the \ea by the number of iterations of the repeat loop. Note that each iteration requires to carry out out one single fitness evaluation. The optimization time refers to the number of fitness evaluations until an optimal solution has been obtained for the first time, and the expected optimization time refers to the expectation of this value.

 \begin{algorithm}[t]
  Choose $x \in \{0,1\}^n$ uniformly at random\;
 \Repeat{$\mathit{stop}$}{
 Create $y$ by flipping each bit $x_{i}$ of $x$ with probability $\frac{1}{n}$\;
 \If{$f(y) \leq f(x)$} {
   $x \leftarrow y;$}
     }
 \caption{\ea for minimization} \label{alg:EA}
 \end{algorithm}

\subsection{Lower Bound for \ea and Uniform Constraint}
\label{subsec:lb}

We consider the (1+1)~EA for the problem stated in Equation~\ref{chance-problem} with an additional uniform constraint that requires that each feasible solution contains at least $k$ elements, i.e.\ $|x|_1 \geq k$ holds. We show that the (1+1)~EA has an exponential optimization time on an even very restrictive class of instances involving only two different weight distributions. 

We use the following fitness function, which should be minimized in the (1+1)~EA:
\begin{eqnarray*}
f(x) = \begin{cases}
g(x) &  |x|_1 \geq k\\
(k-|x|_1)\cdot L & |x|_1<k,
\end{cases}
\end{eqnarray*}
where $L =(1+\sum_{i=1}^n \mu_i+ K_{\alpha} (\sum_{i=1}^n \sigma_i^2)^{1/2})$.
This gives a large penalty to each unit of constraint violation.
It implies that any feasible solution is better than any infeasible solution and that $f(x) > f(y)$ holds if both $x$ and $y$ are infeasible and $|x|_1 > |y|_1$. Furthermore, the fitness value of an infeasible solution only depends on the number of its elements.

We now show a lower bound on the optimization time of the (1+1)~EA for a specific instance class $I$ containing only two types of elements.
Type $a$ elements have weights chosen according to $N(n^2+\delta, 1)$ and type $b$ elements have weights chosen according to $N(n^2, 2)$. We set $\delta= \frac{1}{2\sqrt{k\cdot 1.48}}$.
The instance $I$ has exactly $n/2$ elements of type $a$ and $n/2$ elements of type $b$.
We consider $K_{\alpha}=1$ which matches $\alpha \approx 0.84134$, and set $k=0.51n$. Using the fitness function~$f$, we have the additional property for two feasible solutions $x$ and $y$ that $f(x) < f(y)$ if $|x|_1 < |y|_1$ due to an expected weight of at least $n^2$ for any additional element in a feasible solution. This also implies that an optimal solution has to consist of exactly $k$ elements. The quality of a solution with $k$ elements only depends on the number of type $a$ and type $b$ elements it contains. An optimal solution includes $n/2$ elements of type $a$ whereas a locally optimal solution includes $n/2$ elements of type $b$. Note that an optimal solution has minimal variance among all feasible solutions. The \ea produces with high probability the locally optimal solution before the global optimum, which leads to the following result.

\begin{theorem}
\label{thm:LBoneone}
The optimization time of the (1+1)~EA on the instance $I$ using the fitness function $f$ is $e^{\Omega(n)}$ with probability $1 - e^{-\Omega(n^{1/4})}$.
\end{theorem}

\begin{proof}

We first analyze the quality of solutions with exactly $k$ elements and show that $x^*$ is an optimal solution if it contains exactly $n/2$ elements of type $a$.
Consider a solution $x$ with $\ell$ elements of type $b$ and $k$ elements in total.

We have
$$
f(x) =  (k-\ell)( n^2 + \delta) + \ell \cdot n^2 + \sqrt{k + \ell} =  k (n^2 + \delta) - \delta \ell + \sqrt{k+ \ell}.
$$

Let $\ell \in [\max\{0,k-n/2\}, \min\{k, n/2\}]$ and consider the first derivative of the corresponding continuous function
$$
p(\ell) = k (n^2 + \delta) - \delta \ell + \sqrt{k+ \ell}. 
$$

We have $p'(\ell) = -\delta + \frac{1}{2 \sqrt{k+ \ell}} = - \frac{1}{2\sqrt{k\cdot 1.48}}+ \frac{1}{2 \sqrt{k+ \ell}}$, and $p'(\ell)=0$ for $\ell=0.48\cdot k$ and $p''(\ell)= - \frac{1}{4(k+\ell)^{3/2}}$ which is negative for all $\ell \geq 0$. 

Therefore, $p$ has its maximum at $\ell = 0.48 k$, is strongly monotonically increasing in $\ell$  in the interval $[0, 0.48\cdot k[$, and strongly monotonically decreasing for $\ell \in ]0.48 \cdot k, k]$.
We have $k=0.51n$ with implies that we consider integer values of $\ell \in [0.01n, 0.5n]$.
We have 
\begin{eqnarray*}
& & p(0.01n) < p(0.5n) \\
& \Longleftrightarrow & k(n^2+ \delta) - 0.01 \delta n + \sqrt{0.52n} < k(n^2+ \delta) - 0.5 \delta n + \sqrt{1.01n}\\
& \Longleftrightarrow & 0.49n \cdot \frac{1}{2\sqrt{0.51n \cdot 1.48}} < \sqrt{1.01n}-\sqrt{0.52n}\\
& \Longleftrightarrow &  \frac{0.49}{2 \cdot \sqrt{0.7548}} \cdot \sqrt{n}  <  (\sqrt{1.01} - \sqrt{0.52}) \sqrt{n}
\end{eqnarray*}
where the last inequality holds as $\frac{0.49}{2 \cdot \sqrt{0.7548}}< 0.2821$ and $\sqrt{1.01} - \sqrt{0.52}>0.2838$.

Hence, among the all solutions consisting of exactly $k$ elements, solutions are optimal if they contain exactly $n/2$ elements of type $a$ and are locally optimal if they contain exactly $n/2$ elements of type $b$.

Any search point with $r+1$ elements is worse than any search point with $r$ elements, $k \leq r \leq n$, as the expected weight of an element is in $\{n^2, n^2+\delta\}$ and the standard deviation of a solution is at most $\sqrt{2n}$ and $K_{\alpha}=1$ holds. Therefore, once a feasible solution has been obtained the number of elements can not be increased and the number of elements decreases until a solution with $k$ elements has been obtained for the first time.

We now show that the (1+1)~EA produces a locally optimal solution consisting of $n/2$ elements of type $b$ within $O(n^2)$ iterations with probability $1-e^{-\Omega(n^{1/4})}$. We do this by considering different phases of the algorithm of length $O(n^2)$ and bounding the failure probabilities by $e^{-\Omega(n^{1/4})}$ in each phase.
The initial solution $x$ chosen uniformly at random is infeasible and we have $0.49n < |x|_1 < 0.51n$ with probability $1-e^{-\Omega(n)}$ using Chernoff bounds. 
For two infeasible solutions $x$ and $y$ we have $f(x) > f(y)$ iff  $|x|_1 > |y|_1$. 
The number of elements in an infeasible solution is increased with probability at least $(n-k)/en=\Theta(1)$ and at most $0.02n$ steps increasing the number of elements are necessary to obtain a feasible solution. Hence, and expected number of $O(n)$ is required to produce a feasible solution and the probability to achieve a feasible solution within $O(n^2)$ steps is $1 - e^{-\Omega(n)}$ using Markov's inequality in combination with a restart argument consisting of $n$ restarts.

Until a feasible solution has been reached for the first time, there is no bias and the first feasible solution with $r\geq k$ items is chosen uniformly at random among all sets containing exactly $r$ items. 
When reaching the first feasible solution with $r\geq k$ elements, it contains in expectation $r/2$ elements of type $a$ and $r/2$ elements of type $b$. The probability that the number of type $a$ elements is less than $0.499r$ is $e^{-\Omega(n)}$ using the Chernoff bound for the hypergeometric distribution (see
Theorem~1.10.25 in~\citealp{DoerrProbabilisticTools}).
We have $r \leq k + n^{1/2}$ with probability $1 - e^{-\Omega(n^{1/2})}$ for the first feasible solution as the probability to flip at least $n^{1/2}$ in a single mutation step within $O(n^2)$ steps is $e^{-\Omega(n^{1/2})}$. 

Using arguments as before, the function
\[p_r(\ell) = r (n^2 + \delta) - \delta \ell + \sqrt{r+ \ell}\]
has its maximum at $0.48r$ and is strongly monotonically increasing in $\ell$  in the interval $[0, 0.48\cdot r[$, and strongly monotonically decreasing for $\ell \in ]0.48 \cdot r, r]$.
Hence producing from a solution $x$ an solution $y$ with $|x|_1 = |y|_1$ with a smaller number of type $b$ elements is not accepted unless a solution with at most $0.48r$ elements of type $b$ is produced which happens with probability $e^{-\Omega(n)}$ in $\Theta(n^2)$ steps. The only steps that may reduce the number of type $b$ elements are steps reducing the number of elements in the solution overall. A step reducing the number of elements has probability at least $r/(en)=\Theta(1)$. At most $n^{1/2}$ of such steps are required to obtain a solution with exactly $k$ elements and the expected number of steps required to produce from a solution with $r \leq k+n^{1/4}$ elements a solution with $k$ elements is $O(n^{1/4})$ and a solution with exactly $k$ elements has been produced in time $O(n^2)$ with probability $1-e^{-\Omega(n)}$ using again Markov's inequality in combination with $n$ restarts.
In each of these $O(n^2)$ steps, there is no mutation flipping at least $n^{1/4}$ bits with probability $e^{-\Omega(n^{1/4})}$. This implies
that the algorithm reaches a solution with $k$ elements where at least 
$$0.499\cdot(k + n^{1/2})- n^{1/4}\cdot n^{1/2} \geq 0.49k$$ elements are of type $b$ with probability $1-e^{-\Omega(n^{1/4})}$ within a phase of $O(n^2)$ steps.

Having obtained a solution with $k$ elements and at least $0.49k$ elements of type~$b$, the (1+1)~EA does not accept any solution with less elements of type $b$ unless a solution with at most $0.48k$ elements of type $b$ is produced.
Producing a solution with at most $0.48k$ elements of type $b$ requires flipping at least $0.01k$ bits which happens with probability $e^{-\Omega(n)}$ in $O(n^2)$ steps.
Let $x$ be a solution with $k$ elements where $\ell<n/2$ elements are of type $b$. The probability to produce from $x$ a solution with $k$ elements and at least $\ell+1$ elements of type $b$ is
at least 
$$(n/2- \ell)\cdot (k-\ell)/(en^2) \geq (n/2- \ell) \cdot 0.01n/ (en^2) = (n/2- \ell)/(100 en)$$
as $k- \ell \geq 0.01n$ and a type $a$ and a type~$b$ elements need to swap. 
The number of type~$b$ elements is increased until having achieved $\ell=n/2$ in expected time $O(n \log n)$ using fitness based partitions~\citep{Jansen13} with respect to $\ell$ and summing up the waiting times to increase the different values of $\ell$. Using again Markov's inequality together with the restart argument involving $n/\log n$ restarts, the probability to have not obtained the solution with $n/2$ elements of type $b$ is $e^{-\Omega(n/\log n)}$. 
Hence, the local optimal solution with $k$ elements among them $n/2$ of type $b$ is produced in $O(n^2)$ steps with probability $1-e^{-\Omega(n^{1/4})}$.

The time to producing from such a locally optimal solution an optimal solution consisting of $n/2$ elements of type $a$ is $e^{\Omega(n)}$  with probability $1- e^{-\Omega(n)}$ as the number of type $b$ elements needs to be reduced to at most $0.48k$ in order to accept an offspring with less than $n/2$ elements of type $b$. This implies that the optimization time on the instance $I$ is $e^{\Omega(n)}$ with probability $1-e^{-\Omega(n^{1/4})}$.
\end{proof}

A similar lower bound for a specific class of instances of the chance constrained minimum spanning tree problem having two types of edge weights can be obtained following the ideas given in the proof of Theorem~\ref{thm:LBoneone}.

We now show the asymptotic behaviour of the (1+1)~EA on instance $I$ through an experimental study. Table~\ref{tab:worst-case} shows for $n \in\{100,200,500,1000,1500, 2000\}$ the number of times out of 30 runs the globally optimal solution has been obtained before the locally optimal one. Note that it takes exponential time to escape a locally optimal solution. It can be observed that the fraction of successful runs obtaining the global optimum clearly decreases with $n$. For $n=2000$ no globally optimal solution is obtained within 30 runs.
\begin{table}[t]
\small
    \centering
    \begin{tabular}{|c|c|}
    \hline
       n &  Result (1+1)~EA\\ \hline
      100   & 10/30 \\ 
      200   & 11/30 \\ 
      500   & 6/30 \\ 
      1000   & 3/30 \\ 
      1500   & 1/30  \\ 
      2000   & 0/30 \\ \hline
    \end{tabular}
    \caption{Success rate for the (1+1)~EA on the worst case instance $I$.}
   \label{tab:worst-case} 
\end{table}

\section{Multi-Objective Evolutionary Algorithm}
\label{sec:MOEA}

We now introduce bi-objective formulations of the chance constrained problems with uniform and spanning tree constraints.
We use a Pareto Optimisation approach for this which computes trade-offs with respect to the expected weight~$\mu$ and variance~$v$.
We say that a solution $z$ dominates a solution $x$ (denoted as $z \preccurlyeq x$)  iff $\mu(z) \leq \mu(x)$ and $v(z) \leq v(x)$. We say that $z$ strongly dominates x (donated as $z \prec x$) iff $z \preceq x$ and $\mu(z)< \mu(x)$ or $v(z)< v(x)$.

We investigate the algorithm \gsemo~\citep{DBLP:journals/tec/LaumannsTZ04,1299908} shown in Algorithm~\ref{alg:GSEMO}, which has been frequently used in theoretical studies of Pareto optimization. It starts with a solution chosen uniformly at random and keeps at each time step a set of non dominated solutions found so far. In addition to being able to achieve strong theoretical guarantees~\citep{DBLP:journals/ec/FriedrichHHNW10,DBLP:journals/ec/FriedrichN15,DBLP:books/sp/ZhouYQ19}, \gsemo using different types of multi-objective formulation has shown strong performance in practice~\citep{DBLP:conf/nips/QianYZ15,DBLP:conf/ijcai/QianSYT17,DBLP:conf/aaai/RoostapourN0019}. We study the multi-objective evolutionary algorithms in terms of the expected time (measured in terms of iterations of the algorithm) until they have produced a population which contains an optimal solution for each $\alpha \in [1/2,1[$.

\begin{algorithm}[th]
 Choose $x \in \{0,1\}^n$ uniformly at random\;
 $P\leftarrow \{x\}$\;
\Repeat{$\mathit{stop}$}{
Choose $x\in P$ uniformly at random\;
Create $y$ by flipping each bit $x_{i}$ of $x$ with probability $\frac{1}{n}$\;
\If{$\nexists\, w \in P: w \prec y$} {
  $P \leftarrow (P \setminus \{z\in P \mid y \preceq z\}) \cup \{y\}$\;
  }
    }
\caption{Global SEMO} \label{alg:GSEMO}
\end{algorithm}

\subsection{Uniform Constraints}
\label{subsec:mo}
For the case of the uniform constraint $|x|_1 \geq k$, we consider the objective function 
$f(x) = (\mu(x), v(x))$
where 
\begin{eqnarray*}
\mu(x) = \begin{cases}
 \sum_{i=1}^n \mu_i x_i & |x|_1 \geq k\\
(k-|x|_1)\cdot (1+\sum_{i=1}^n \mu_i) & |x|_1<k
\end{cases}
\end{eqnarray*}

\begin{eqnarray*}
v(x) = \begin{cases}
 \sum_{i=1}^n \sigma^2_i x_i & |x|_1 \geq k\\
(k-|x|_1)\cdot (1+\sum_{i=1}^n \sigma^2_i) & |x|_1<k
\end{cases}
\end{eqnarray*}

Note that it gives the expected value and variance for any feasible solution, and a large penalty for any unit of constraint violation in each objective function if a solution is infeasible. 
This implies that the objective value of an infeasible solution is always worse than the value of a feasible solution.

As we have $\mu_i \geq 1$ and $\sigma_i \geq 1$, $1 \leq i \leq n$, each Pareto optimal solution contains exactly $k$ elements. This is due to the fact that we can remove from a solution $x$ with $|x|_1 >k$ any element to obtain a solution $y$ with $\mu(y) < \mu(x)$ and $v(y) < v(x)$.
We will minimize 
$f_{\lambda}(x) = \lambda \mu(x) + (1- \lambda) v(x)$
by selecting minimal elements with respect to
$f_{\lambda}(e_i) = \lambda \mu_i + (1- \lambda) \sigma_i^2$, $0 < \lambda < 1$.
For the special cases $\lambda=0$ and $\lambda=1$, we minimize $f_{\lambda}$ by minimizing $f_0(x)=(v(x), \mu(x))$ and $f_1(x)=(\mu(x), v(x))$ with respect to the lexicographic order. Note, that we are using $f_{\lambda}$ both for the evaluation of a search point $x$ as well as the evaluation of an element $e_i$.
For each fixed $\lambda \in [0,1]$, an optimal solution for $f_{\lambda}$ can be obtained by sorting the items increasing order of $f_{\lambda}$ and selecting the first $k$ of them.
For a given set $X$ of such points we denote by $X^*_{\lambda} \subseteq X$ the set of minimal elements with respect to $f_{\lambda}$. Note that all points in the sets $X^*_{\lambda}$, $0 \leq \lambda \leq 1$, are not strongly dominated in $X$ and therefore constitute Pareto optimal points when only considering the set of search points in $X$.

\begin{definition}[Extreme point of set $X$]
\label{def:extremepoint}
For a given set $X$, we call $f(x) =(\mu(x),v(x))$ an extreme point of $X$ if there is a $\lambda \in [0,1]$ such that $x \in X^*_{\lambda}$ and $v(x) = \max_{y \in X^*_{\lambda}} v(x)$.
\end{definition}

We denote by $f(X)$ the set of objective vectors corresponding to a set $X\subseteq 2^E$, and by $f(2^E)$ the set of all objective vectors of the considered search space $2^E$. The extreme points of $2^E$ are given by the extreme points of $f(2^E)$.
A crucial property of the extreme points is that they contain all objective vectors that are optimal for any $\lambda \in [0,1]$. Hence, if there is an optimal solution that can be obtained by minimizing $f_{\lambda}$ for a (potentially unknown) value of $\lambda$, then such a solution is contained in the set of search points corresponding to the extreme points of $2^E$.

In the following, we relate an optimal solution of
\begin{eqnarray*}
g(x) = \sum_{i=1}^n \mu_i x_i + K_{\alpha} \cdot\left(\sum_{i=1}^n \sigma_i^2 x_i\right)^{1/2}
\text{\quad subject to } |x|_1 \geq k
\end{eqnarray*}
to an optimal solution of 
\begin{eqnarray*}
g_R(x) = R \cdot \sum_{i=1}^n \mu_i x_i + K_{\alpha} \cdot\left(\sum_{i=1}^n \sigma_i^2 x_i \right) 
\text{\quad subject to } |x|_1 \geq k
\end{eqnarray*}
for a given parameter $R\geq0$ that determines the weightening of $\mu(x)$ and $v(x)$. Note that $g_R$ is a linear combination of the expected value and the variance and optimizing $g_R$ is equivalent to optimizing $f_{\lambda}$ for $\lambda=R/(R+K_{\alpha})$  as we have $g_R(x) = (R + K_{\alpha}) \cdot f_{\lambda}(x)$ in this case. We use $g_R$ to show that there is a weightening that leads to an optimal solution for $g$ following the notation given in \cite{DBLP:journals/dam/IshiiSNN81}, but will work with the normalized weightening of $\lambda$ when analyzing our multi-objective approach. 

Let $x^*$ be an (unknown) optimal solution for $g$ and let $D(x^*) = \left(\sum_{i=1}^n \sigma_i^2 x^*_i \right)^{1/2}$ be its standard deviation. 
Lemma~\ref{lem:optgD} follows directly from the proof of Theorems 1--3 in \cite{DBLP:journals/dam/IshiiSNN81} where it has been shown to hold for the constraint where a feasible solution has to be a spanning tree. However, the proof only uses the standard deviation of an optimal solution and relates this to the weightening of the expected value and the variance. It therefore holds for the whole search space independently of the constraint that is imposed on it.
Therefore, it also holds for the uniform constraint where we require $|x|_1 \geq k$.
\begin{lemma}[follows from Theorems 1--3 in \citealp{DBLP:journals/dam/IshiiSNN81}]
\label{lem:optgD}
An optimal solution for $g_{2D(X^*)}$ is also optimal for $g$.
\end{lemma}

Based on Lemma~\ref{lem:optgD}, an optimal solution for $f_{\lambda}$, where $\lambda= 2D(X^*)/(2D(X^*)+K_{\alpha})$, is also optimal for $g$. As we are dealing with a uniform constraint, an optimal solution for $f_{\lambda}$ can be obtained by greedily selecting elements according to $f_{\lambda}$ until $k$ elements have been included. The extreme points allow to cover all values of~$\lambda$ where optimal solutions differ as they constitute the values of $\lambda$ where the optimal greedy solution might change and we bound the number of such extreme points in the following.

In order to identify the extreme points of the Pareto front, we observe that the order of two elements $e_i$ and $e_j$ with respect to a greedy approach selecting always a minimal element with respect to $f_{\lambda}$ can only change for one fixed value of $\lambda$.
We define $\lambda_{i,j} = \frac{\sigma_j^2 - \sigma_i^2}{(\mu_i-\mu_j) +(\sigma^2_j - \sigma^2_i)}$ for the pair of items $e_i$ and $e_j$ where $\sigma^2_i < \sigma^2_j$ and $\mu_i > \mu_j$ holds, $1 \leq i < j \leq n$.

Consider the set $\Lambda=\{\lambda_0, \lambda_1, \ldots, \lambda_{\ell}, \lambda_{\ell+1}\}$ where $\lambda_1, \ldots, \lambda_{\ell}$ are the values $\lambda_{i,j}$ in increasing order and 
$\lambda_0=0$ and $\lambda_{\ell+1}=1$. The key observation is that computing Pareto optimal solutions that are optimal solutions for $f_{\lambda}$ and every $\lambda \in \Lambda$ gives the extreme points of the problem. 

\begin{lemma}
\label{lem:number-expoints}
Each extreme point of the multi-objective formulation is Pareto optimal and optimal  with respect to $f_{\lambda}$ for at least one $\lambda \in \Lambda$. 
The number of extreme points is at most $n(n-1)/2+2 \leq n^2$.
\end{lemma}

\begin{proof}
For each extreme point $f(x)$ there is a $\lambda \in [0,1]$ such that $x$ is optimal for $f_{\lambda}$ which implies that $x$ is Pareto optimal.
Let $X^* \subseteq 2^E$ be the set of solutions corresponding to the extreme points of the Pareto front.
For each extreme point $f(x^*) \in f(X^*)$, there exists a $\lambda^* \in [0,1]$ such that 
$f_{\lambda^*}(x^*) = \min_{x \in 2^E} f_{\lambda^*}(x)$ and $f_{\lambda^*}(y) > \min_{x \in 2^E} f_{\lambda^*}(x)$ for all $f(y) \not =f(x^*)$.

If an item $e_i$ dominates an item $e_j$, i.e. we have $\mu_i \leq \mu_j$ and $\sigma^2_i \leq \sigma^2_j$, then $e_i$ can be included prior to $e_j$ in a greedy solution for any $f_{\lambda}$.
Assume that the item $e_1, \ldots, e_n$ are sorted in increasing order of the variances $\sigma^2_i$, i.e $i\leq j$ iff $\sigma^2_i \leq \sigma^2_j$.
For two items $e_i$ and $e_j$, $i < j$, that are incomparable, i.e. where $\sigma^2_i < \sigma^2_j$ and $\mu_i > \mu_j$ holds, the order of the items for $f_{\lambda}$ changes from $f_{\lambda}(e_i) < f_{\lambda}(e_j)$ to $f_{\lambda}(e_j) < f_{\lambda}(e_i)$ at exactly one particular threshold value $\lambda$.

Consider the items $e_i$ and $e_j$ with $\sigma^2_i < \sigma^2_j$ and $\mu_i > \mu_j$. 
We have $f_{\lambda}(e_i)< f_{\lambda}(e_j)$ iff
\begin{eqnarray*}
& & \lambda \mu_i + (1-\lambda) \sigma^2_i < \lambda \mu_j + (1-\lambda) \sigma^2_j\\
& \Longleftrightarrow & \lambda /(1-\lambda) < (\sigma^2_j - \sigma^2_i)/(\mu_i-\mu_j)\\
& \Longleftrightarrow & \lambda < \frac{\sigma_j^2 - \sigma_i^2}{(\mu_i-\mu_j) +(\sigma^2_j - \sigma^2_i)},
\end{eqnarray*}.

Moreover, we have 
$$f_{\lambda}(e_i)= f_{\lambda}(e_j) \text{ iff } \lambda = \frac{\sigma_j^2 - \sigma_i^2}{(\mu_i-\mu_j) +(\sigma^2_j - \sigma^2_i)}$$
and 
$$f_{\lambda}(e_i)> f_{\lambda}(e_j) \text{ iff } \lambda > \frac{\sigma_j^2 - \sigma_i^2}{(\mu_i-\mu_j) +(\sigma^2_j - \sigma^2_i)}.$$

We define $\lambda_{i,j} = \frac{\sigma_j^2 - \sigma_i^2}{(\mu_i-\mu_j) +(\sigma^2_j - \sigma^2_i)}$ for the pair of items $e_i$ and $e_j$ where $\sigma^2_i < \sigma^2_j$ and $\mu_i > \mu_j$ holds, $1 \leq i < j \leq n$.  Note $\lambda_{i,j} \in [0,1]$ if $\sigma^2_i < \sigma^2_j$ and $\mu_i > \mu_j$, $1 \leq i < j \leq n$, and that these values of $\lambda$ are the only values where the order between the two elements according to $f_{\lambda}$ can change. This implies that these are the only weightenings where the greedy solution (which is an optimal solution) may change.

Consider the set $\Lambda=\{\lambda_0, \lambda_1, \ldots, \lambda_{\ell}, \lambda_{\ell+1}\}$ where $\lambda_1, \ldots, \lambda_{\ell}$ are the values $\lambda_{i,j}$ in increasing order and 
$\lambda_0=0$ and $\lambda_{\ell+1}=1$.
We have $\ell \leq n(n-1)/2$ as we only need to consider pairs of items and therefore $|\Lambda| \leq n(n-1)/2+2 \leq n^2$.
\end{proof}

In the following, we assume that $v_{\max} \leq \mu_{\max}$ holds. Otherwise, the bound can be tightened by replacing $v_{\max}$ by $\mu_{\max}$.
The following lemma gives an upper on the expected time until \gsemo has obtained Pareto optimal solution of minimal variance. Note that this solution is optimal for $f_0$. We denote by $P_{\max}$ the maximum population size that \gsemo encounters during the run of the algorithm.

\begin{lemma}
\label{lem:MO-minvariance}
The expected time until \gsemo has included a Pareto optimal search point of minimal variance in the population is $O(P_{\max}n^2(\log n+ \log v_{\max}))$.
\end{lemma}

\begin{proof}
Starting with an infeasible solution this implies that the population size is $1$ until a feasible solution has been obtained for the first time. The fitness of an infeasible solution is determined by the number of elements $r=|x|_1<k$ in the solution. It increases with probability at least $r/(en)$ in the next iteration. As no solution with less then $r$ elements is accepted until a feasible solution has been obtained for the first time, the number of elements increases to at least $k$ within $O(n \log n)$ steps using a fitness level argument (see \citealp{Jansen13}) with respect to the number of elements in an infeasible solution.

Once a feasible solution has been obtained for the first time, only feasible solutions are accepted during the optimization process.
We use multiplicative drift analysis \citep{DoerrJohannsenWinzenALGO12} and consider in each step the search point $x$ with the smallest value of $v(x)$ in the population. Let $r=|x|_1\geq k$. If $r>k$, then flipping any of the 1-bits is accepted and the value of $v(x)$ decreases at least by a factor of $1-1/(P_{\max}en)$ in expectation in the next iteration.
If $r=k$, consider a Pareto optimal solution $x^*$ that has minimal variance. Assume that there are $s\geq 1$ elements contained in $x$ that are not contained in $x^*$. Then removing any of these elements from $x$ and including one of the missing $s$ elements from $x^*$ is accepted as otherwise $x^*$ can not be a Pareto optimal solution of minimal variance. Furthermore, there are $s$ such operations that are all accepted and in total reduce the variance from $v(x)$ to $v(x^*)$.
Hence, the variance of the solution having the smallest variance in the population reduces by  at least $1/(P_{\max}en^2)\cdot(v(x) -v(x^*))$ in expectation in the next iteration. Note, that this reduction is less than the one in the case of $r>k$. Using the  multiplicative drift theorem \citep{DoerrJohannsenWinzenALGO12}, a Pareto optimal solution of minimal variance is obtained in expected time $O(P_{\max}n^2(\log n+ \log v_{\max}))$.
\end{proof}

Lemma~\ref{lem:MO-minvariance} directly gives an upper bound for \gsemo on the worst case instance $I$ for which the (1+1)~EA has an exponential optimization time.
The feasible solution of minimal variance is optimal for instance $I$. Furthermore, we have $v_{\max} \leq 2n$ and the variance can only take on $O(n)$ different values. which implies that $P_{\max}=O(n)$ holds for instance~$I$. Therefore, we get the following result for \gsemo on the worst case instance $I$ of the (1+1)~EA. 
\begin{theorem}
The expected time until \gsemo has obtained an optimal solution for the instance$I$ is $O(n^3 \log n)$.
\end{theorem}

We now present the result for the general class with a uniform constraint.
Based on a Pareto optimal solution having the minimal variance, \gsemo can construct all other extreme points and we obtain the following result.

\begin{theorem}
\label{theo:gsemo-optimal-unif}
Considering the chance constrained problem with a uniform constraint, the expected time until \gsemo has computed a population which includes an optimal solution for any choice of $\alpha \in [1/2,1[$ is $O(P_{\max} n^2 \ell (\log n + \log v_{\max}))$.
\end{theorem}

\begin{proof}
We assume that we have already included a Pareto optimal solution of minimal variance $v_{\min}$ into the population.
Let 
$$v_{\lambda}^{\max}= \max_{x \in 2^E} \left\{ v(x) \mid f_{\lambda}(x) = \min_{z \in 2^E} f_{\lambda}(z) \right\}$$ 
and 
$$v_{\lambda}^{\min}= \min_{x \in 2^E} \left\{ v(x) \mid f_{\lambda}(x) = \min_{z \in 2^E} f_{\lambda}(z) \right\}$$
be the maximal and minimal variance of any optimal solution for the linear weightening $f_{\lambda}$.

Note that we have $v_{\lambda}^{\max} = v_{\lambda}^{\min}$ for $\lambda=0$ as the Pareto optimal objective vector of minimal variance is unique.
Hence, the Pareto optimal solution of minimal variance $v_{\min}=v_{0}^{\max}$ is a solution of maximal variance for $\lambda=0$.
Consider $\lambda_i$, $0 \leq i \leq \ell$. We have $v_{\lambda_i}^{\max}= v_{\lambda_{i+1}}^{\min}$ as the extremal point that is optimal for $f_{\lambda_i}$ and $f_{\lambda_{i+1}}$ has the largest variance for $f_{\lambda_i}$ and the smallest variance for $f_{\lambda_{i+1}}$ among the corresponding sets of optimal solutions.

Assume that we have already included into the population a search point $x$ that is minimal with respect to $f_{\lambda_i}$ and has maximal variance $v_{\lambda_i^{\max}}$ among all these solutions.
The solution $x$ is also optimal with respect to $f_{\lambda_{i+1}}$ and we have $v_{\lambda_i}^{\max}= v_{\lambda_{i+1}}^{\min}$. We let $r$ be the number of elements contained in $x$ but not contained in the optimal solution $y$ for $f_{\lambda_{i+1}}$ that has variance $v_{\lambda_{i+1}^{\max}}$. As both solutions contain $k$ elements, differ by $r$ elements and $y$ has maximal variance with respect to $f_{\lambda_{i+1}}$, there are $r^2$ 2-bit flips that bring down the distance $d(x) = v_{\lambda_{i+1}}^{\max} - v(x) \leq v_{\lambda_{i+1}}^{\max} - v_{\lambda_{i}}^{\max}$. Using the multiplicative drift theorem~\citep{DoerrJohannsenWinzenALGO12} where we always choose the solution that is optimal with respect $f_{\lambda_{i+1}}$ and has the largest variance, the expected time to reach such as solution of variance $v_{\lambda_{i+1}^{\max}}$ is $O(P_{\max}n^2 \log (v_{\lambda_{i+1}}^{\max} - v_{\lambda_{i}}^{\max}))=O(P_{\max}n^2 (\log n + \log v_{\max}))$.
Summing up over the different values of $i$, we get $O(P_{\max}n^2 \ell (\log n + \log v_{\max}))$ as an upper bound on the expected time to generate all extreme points.
\end{proof}

\subsection{Chance Constrained Minimum Spanning Trees}

We now extend the previous results to the chance constrained minimum spanning tree problem 
where edge weights are independent and chosen according to a normal distribution.  
Note that using the expected weight and the variance of a solution as objectives results in a bi-objective minimum spanning tree problems for which a runtime analysis of \gsemo has been provided in \cite{DBLP:journals/eor/Neumann07}.

Let $c(x)$ be the number of connected components of the solution $x$. 
We consider the bi-objective formulation for the multi-objective minimum spanning tree problem given in \cite{DBLP:journals/eor/Neumann07}. Let $w_{ub} = n^2\cdot \max\{\mu_{\max}, v_{\max}\}$. 
The fitness of a search point $x$ is given as 
$$f(x) = (\mu(x), v(x))$$ where 
$$\mu(x) = (c(x)-1)\cdot w_{ub} +  \sum_{i=1}^m \mu_i x_i$$
and 
$$v(x) = (c(x)-1)\cdot w_{ub} +  \sum_{i=1}^m \sigma^2_i x_i.$$
It gives a large penalty for each additional connected component.
We transfer the results for the multi-objective setting under the uniform constraint to the setting where a feasible solution has to be a spanning tree. The crucial observation from \cite{DBLP:journals/eor/Neumann07} to obtain the extreme points is that edge exchanges resulting in new spanning trees allow to construct solutions on the linear segments between two consecutive extreme points in the same way as in the case of the uniform constraint. 
Let $\ell \leq m(m-1)/2$ be the pairs of edges $e_i$ and $e_j$ with $\sigma_i^2 < \sigma_j^2$ and $\mu_i > \mu_j$.
Similar to Lemma~\ref{lem:number-expoints} and using the arguments in \cite{DBLP:journals/dam/IshiiSNN81}, the number of extreme points is at most $\ell+2 \leq m^2$ as an optimal solution for $f_{\lambda}= \lambda \mu(x) + (1-\lambda) v(x)$ can be obtained by Kruskal's greedy algorithm.
We replace the expected time of $O(n^2)$ for an items exchange in the case of the uniform constraint with the expected waiting time of $O(m^2)$ for a specific edge exchange in the case of the multi-objective spanning tree formulation and get the following results.

\begin{theorem}
\label{theo:gsemo-optimal-mst}
Considering the chance constrained minimum spanning tree problem, the expected time until \gsemo has computed a population which includes an optimal solution for any choice of $\alpha \in [1/2,1[$ is $O(P_{\max} m^2 \ell (\log n + \log v_{\max}))$.
\end{theorem}

\section{Convex Multi-Objective Evolutionary Algorithms}
\label{sec:convex-MOEA}
The results presented in the previous section depended on $P_{\max}$, which might be exponential with respect to the given input in the case that the expected values and variances are exponential in $n$. We now introduce approaches which provably can obtain the set of optimal solutions
in expected polynomial time.
\begin{algorithm}[t]
 Given a set of search points $P$ and a bi-objective objective function $f=(f_1, f_2)$\;
  Set $r:=1$, $S:=P$\;
     \Repeat{$S=\emptyset$}{
    Let $x$ and $y$ be solutions that are minimal in $S$ with respect to the lexicographic order $(f_1,f_2)$ and $(f_2, f_1)$, respectively.\;
     Compute the convex hull $C(S)$ between $x$ and $y$.\;
     Set $r_P(x) = r$ for all $x \in C(S)$\;
     $S:= S \setminus C(S)$.\;
     $r:=r+1$\;
      }
 \caption{Convex Hull Ranking Algorithm~\citep{DBLP:journals/scira/MonfaredM011} \label{alg:CR}}
\end{algorithm}

As the extreme points are of significant importance when computing optimal solutions for the chance constrained problems studied in this paper, we consider a convex evolutionary algorithm approach similar to the convex approach examined in \cite{DBLP:journals/scira/MonfaredM011}. It makes use of the convex hull ranking algorithm shown in Algorithm~\ref{alg:CR}, which has originally been introduced in \cite{DBLP:journals/scira/MonfaredM011}. Our Convex $(\mu+1)$-EA shown in Algorithm~\ref{alg:CV-EA} uses the convex hull ranking algorithm to assign ranks to the search points in a given population and use this for survival selection. 
Algorithm~\ref{alg:CR} assigns ranks to individuals by iterative computing the convex hull of a population in the objective space and assigning ranks to the individuals when they become part of the convex hull of the remaining set of elements.

For a given set of elements $S$, Algorithm~\ref{alg:CR} selects the element $x$ that is minimal w.\,r.\,t.\ the lexicographic order of $(f_1, f_2)$ and $y$ which is minimal with respect to the lexicographic order of $(f_2, f_1)$. Note that $x$ and $y$ are not strictly dominated by any solution in $S$ and are solutions of minimal value with respect to $f_1$ and $f_2$, respectively. Let $C(S)$ be the extreme points of $S$ which defines the convex hull of $S$.
As $x$ and $y$ are such extreme points, they are part of the convex hull points $C(S)$ of the given set $S$ in the objective space. The convex ranking algorithm initially sets the rank value~$r$ to~$1$. It then adds $x$, $y$, and all other points on the convex hull to $C(S)$. Note that this step can be done in time $O(|S| \log |S|)$ by using the Graham scan algorithm. 
It assigns to all points in the $C(S)$ the same rank $r$, and removes $C(S)$ from $S$ afterwards. It then proceeds by increasing the rank $r$ by $1$ and repeating the procedure until the set $S$ is empty.

\begin{algorithm}[t]
 Choose a population $P$ consisting of $\mu$ individuals from $\{0,1\}^n$ uniformly at random\;
 $P\leftarrow \{x\}$\;
\Repeat{$\mathit{stop}$}{
Choose $x\in P$ uniformly at random\;
Create $y$ by flipping each bit $x_{i}$ of $x$ with probability $\frac{1}{n}$\;
$P:=P \cup \{y\}$\;
Remove an individual $z$ from $P$ with $z=\arg \max_{x \in P} (r_P(x), v(x))$
    }
\caption{Convex ($\mu+1$)-EA} \label{alg:CV-EA}
\end{algorithm}

\begin{algorithm}[t]
 $P\leftarrow \{x\}$\;
\Repeat{$\mathit{stop}$}{
Choose $x\in P$ uniformly at random\;
Create $y$ by flipping each bit $x_{i}$ of $x$ with probability $\frac{1}{n}$\;
\If{$y \in C(P \cup \{y\})$}{
$P \leftarrow (P \setminus \{z\in P \mid y \preceq z\}) \cup \{y\}$\;
$P \leftarrow P \setminus (P \setminus C(P))$\;
 \If{$|P|> P_{ub}$}{
 $z = \arg \max_{x \in P} v(x)$\;
 $P \leftarrow P \setminus \{z\}$\;
}
}
}
\caption{Convex \gsemo (with upper bound $P_{ub}$ on population size)} \label{alg:CO-GSEMO}
\end{algorithm}

We denote by $r_P(x)$ the rank of solution~$x$ with respect to population $P$ using the convex hull ranking algorithm. 
When removing an element $z$ from $P$ in Algorithm~\ref{alg:CV-EA}, then this element is chosen according to the lexicographic order with respect to $(r_P(x), v(x))$, i.e.\ an element with the highest variance from the front having the largest rank is removed.

Let $\lambda^* \in [0,1]$ and $X$ be an arbitrary set of search points. Then the element $x^* \in X^*_{\lambda}$ with $x^* = \arg \max_{x \in X^*} v(x)$ is an extreme point of the convex hull of $X$ according to Definition~\ref{def:extremepoint}.
We consider the different values of $\lambda \in \Lambda$, where extreme points can change based on the order of pairs of edges. The following lemma shows that extreme points that are obtained with increasing 
$\lambda$ stay in the population.  

\begin{lemma}
Assume $\mu \geq \ell+2$ and that $P$ includes the extreme points for $\lambda_0, \ldots, \lambda_q \in \Lambda$. Then the algorithm does not remove any of these extreme points during the run of the algorithm.
\end{lemma}

\begin{proof}
The extreme points all have rank $1$ as otherwise they will not be extreme points for the given problem. Furthermore, all solutions of rank $1$ with variance $v(x) \leq v_{\lambda_q}^{\max}$ are extreme points and we have $q+1 \leq \ell+2 \leq \mu$ of them. This implies that the algorithm does not remove any of these extreme points in the survival selection step.
\end{proof}

As the Convex $(\mu+1)$-EA does not remove the extreme points and is able to construct extreme points in the same way as \gsemo, it can produce the extreme points for values $\lambda \in \Lambda$ iteratively for increasing values of $\lambda$. 
We therefore get the following result for the case of a uniform constraint.

\begin{theorem}
\label{thm:convex-uniform}
Let $n^2 \geq \mu \geq \ell+2$ and consider the case of a uniform constraint. Then the expected time until the Convex $(\mu+1)$-EA has computed a population which includes an optimal solution for any choice of $\alpha \in [1/2,1[$ is $O(\mu n^2 \ell (\log n + \log v_{\max}))=O(n^4 \ell (\log n + \log v_{\max}))$.
\end{theorem}

Similarly, we get the following polynomial upper bound when considering the multi-objective formulation for the chance constrained minimum spanning tree problem. 

\begin{theorem}
\label{thm:convex-spanning}
Let $m^2 \geq \mu \geq \ell+2$ and consider the chance constrained minimum spanning tree problem. Then the expected time until the Convex $(\mu+1)$-EA has computed a population which includes an optimal solution for any choice of $\alpha \in [1/2,1[$ is $O(\mu m^2 \ell (\log n + \log v_{\max}))=O(m^4 \ell (\log n + \log v_{\max}))$.
\end{theorem}

The proofs of the previous two theorems are analogous to  Theorem~\ref{theo:gsemo-optimal-unif} and 
 Theorem~\ref{theo:gsemo-optimal-mst}, respectively.
Note that the polynomial runtime bound for the Convex ($\mu+1$)-EA and the  computation of the extreme points hold in general for bi-objective spanning trees. It therefore improves upon the pseudo-polynomial runtime given in~\cite{DBLP:journals/eor/Neumann07} to achieve a 2-approximation for bi-objective minimum spanning trees using evolutionary algorithms.

The idea of using only solutions on the convex hull can directly be incorporated into \gsemo. Instead of maintaining a set of solutions that represents all trade-offs with respect to the given objectives functions, Convex \gsemo (see Algorithm~\ref{alg:CO-GSEMO}) maintains a set of solutions that represents the convex hull of the objective vectors found so far during the run of the algorithm. This is done by by accepting an offspring $y$ only if it is part of the convex hull of $P \cup \{y\}$. If $y$ belongs to the convex hull then $y$ is added to the population $P$ and all solutions that are weakly dominated by $y$ or not on the convex hull of the resulting population $P \cup \{y\}$ are removed. Furthermore Convex \gsemo uses a parameter $P_{ub}$ which gives an upper bound on the population size. If the algorithm exceeds $P_{ub}$ then the individual $z$ on the convex hull with the largest variance $v(z)$-value is removed. Setting $P_{ub}\geq \ell+2$ ensures that all points on the convex hull can be included. 
Adapting the proofs of Theorem~\ref{thm:convex-uniform} and \ref{thm:convex-spanning} leads to the following corollaries which show that Convex \gsemo has a polynomial expected time to compute for each $\alpha \in [1/2,1[$ and optimal solution.

\begin{corollary}
Let $n^2 \geq P_{ub}\geq \ell+2$ and consider the case of a uniform constraint. Then the expected time until Convex \gsemo has computed a population which includes an optimal solution for any choice of $\alpha \in [1/2,1[$ is $O(P_{ub}n^2 \ell (\log n + \log v_{\max}))=O(n^4 \ell (\log n + \log v_{\max}))$.
\end{corollary}

\begin{corollary}
Let $m^2 \geq P_{ub}\geq \ell+2$ and and consider the chance constrained minimum spanning tree problem. Then the expected time until Convex \gsemo  has computed a population which includes an optimal solution for any choice of $\alpha \in [1/2,1[$ is $O(P_{ub}m^2 \ell (\log n + \log v_{\max}))=O(m^4 \ell (\log n + \log v_{\max}))$.
\end{corollary}

For our experimental investigations of the convex evolutionary approach, we will consider Convex \gsemo with $P_{ub}=n^2$ instead of the Convex $(\mu+1)$-EA as this directly allows for a comparison to \gsemo and therefore the benefit and drawbacks of only keeping solutions on the convex hull.

\section{Experimental Investigations}
\label{sec:experiments}

We now present experimental results for the chance constrained version of a classical NP-hard optimization problem. We consider the minimum weight dominating set problem. Given a graph $G=(V,E)$ with weights on the nodes, the goal is to compute a set of nodes $D$ of minimal weight such that each node of the graph is dominated by $D$, i.e.\  either contained in $D$ or adjacent to a node in $D$.
Let $n=|V|$ be the number of nodes in the given graph $G=(V,E)$. To generate the benchmarks, we assign each node $u \in V$ a normal distribution $N(\mu(u), v(u))$ with expected weight $\mu(u)$ and a variance $v(u)$.  We consider for each graph values of $\alpha=1-\beta$ where $\beta \in \{0.2, 0.1, 10^{-2}, 10^{-4}, 10^{-6}, 10^{-8}, 10^{-10}, 0^{-12}, 10^{-14}, 10^{-16}\}$. 
We compare \ea, \gsemo, and Convex \gsemo on various and stochastic settings.
Motivated by our theoretical results for the case of a uniform constraint, we set $P_{ub}=n^2$ for Convex \gsemo, but note that this upper bound has never been reached in any run carried out as part of our experimental investigations.

We investigate the graphs cfat200-1, cfat200-2, ca-netscience, a-GrQC and Erdoes992, which are sparse graphs chosen from the network repository~\citep{nr} in three different stochastic settings. The graphs come from different application areas and are of different size. cfat200-1, cfat200-2 from the DIMACS benchmark set have 200 nodes each whereas ca-netscience from the collaboration network benchmark set has 379 nodes. We also investigate  larger graphs from the collaboration network benchmark set, namely ca-GrQC and Erdoes992 which have 4158 and  6100 nodes, respectively. These graphs are investigated with stochastic settings for the cost of the nodes as follows.

\begin{itemize}
    \item In the \emph{uniform random setting} each weight $\mu(u)$ is an integer chosen independently and uniformly at random in $\{n, \ldots, 2 n\}$ and the variance $v(u)$ is an integer chosen independently and uniformly at random in $\{n^2, \ldots, 2n^2\}$.
    
\item In the \emph{degree-based setting}, we set $\mu(u)= (n + \deg(u))^5/n^4$ where $\deg(u)$ is the degree of node $u$ in the given graph. The variance $v(u)$ is an integer chosen independently and uniformly at random in $\{n^2, \ldots, 2n^2\}$. 
\item In the \emph{negatively correlated setting}, each weight $\mu(u)$ is an integer chosen independently and uniformly at random in $\{0, \ldots, n^2\}$ and we set the variance $v(u)=(n^2-\mu(u))\cdot n^2$. 

\end{itemize}
The relationships between expected cost and variance are set in a way such that the expected cost and the standard deviation are of roughly the same magnitude. This allows to explore different types of settings depend on the choice of $\alpha$ in our experimental investigations.
For the degree-based setting, we amplified the difference in degrees through using a power of $5$ and $4$ in the numerator and dominator, respectively.
In the negatively correlated setting, nodes that differ in expected cost are incomparable. We use this setting to investigate how the population sizes of \gsemo and Convex \gsemo grow and impact their search behaviour.

The objective function is given by Equation~\ref{eq:sumofidentityandsquare} where we consider the chosen nodes together with their weight probability distributions. A solution is feasible if it is a dominating set. As we have $\alpha>0.5$, the expected cost is increased based on the standard deviation of the cost of a solution according to Equation~\ref{eq:sumofidentityandsquare}.
We use the fitness functions from Section~\ref{subsec:lb} for the (1+1)~EA and Section~\ref{subsec:mo}  for GSEMO and have a large penalty term as before for each node in $V$ that is not dominated in the given search point $x$.

For each graph and stochastic setting, we generated 30 instances. 
We allocate a budget of 10 million (10M) fitness evaluations to each run. The goal is here to investigate the quality of the results obtained within a reasonable time budget.
\gsemo and Convex \gsemo obtain their results for each graph and stochastic setting for all values of $\alpha$ within a single run of 10M iterations. We allow the \ea to run 10M iterations for each value of $\alpha$, which effectively allows the \ea to spend 10 times the number of fitness evaluations of \gsemo and Convex \gsemo to obtain its results. We used that advantage of the \ea to show that there is still a significant benefit of the multi-objective approaches compared to the single-objective counterpart even when the \ea is allowed to spend the same number of iterations for every value of $\alpha$ for which the multi-objective approaches compute results in a single run.

For each combination of graph, stochastic setting, and $\alpha$, we obtain $30$ results as described before. Note that the although GSEMO obtains in one run results for each value of $\alpha$, the results for a fixed graph and $\alpha$ combinations are independent of reach other.
The results for the different stochastic settings are shown in Tables~\ref{tab:uniform}, \ref{tab:degree}, an~ \ref{tab:ncor}.
We report on the mean and standard deviation of the result obtained for (1+1)~EA, GSEMO and Convex \gsemo and use the Mann-Whitney test to compute the $p$\nobreakdash-value in order to examine statistical significance.
In our tables, $p_1$ denotes the $p$-value for \ea and \gsemo, $p_2$ denotes the $p$-value for \ea and Convex \gsemo, and $p_3$ denotes the $p$-value for for \gsemo and Convex \gsemo. 
We call a result statistically significant if the $p$-value is at most $0.05$. The smallest mean cost value per row and $p$-values less then $0.05$ are marked in bold in the tables in order to easily compare perform and check significance of the results.

We now examine the results for the uniform random setting given shown in Table~\ref{tab:uniform}. For the graph cfat200-1, it can be observed that \ea is outperformed by \gsemo and Convex \gsemo. Most statistical comparisons of the multi-objective approaches to the \ea are also statistical significant. \gsemo achieves on average smaller cost values for most settings compared to \gsemo but the results are not statistically significant.
The results for cfat200-1 show a similar picture with \gsemo and Convex \gsemo again statistically outperforming the \ea for most settings of $\alpha$. The average cost values for Convex \gsemo are the lowest for all values of $\alpha$. However, the statistical comparison between \gsemo and Convex \gsemo does not show statistical significance.
For the graph ca-netscience, \gsemo and Convex \gsemo are significantly better then \ea for all settings of $\alpha$ but there is no statistical different between the two multi-objective approaches. For the large graphs ca-GrQc and Erdos992, the multi-objective approaches achieve lower average cost values than \ea. However, most results are not statistically significant. There only results that are statistically significant show an advantage of Convex \gsemo over \ea for 5 settings of $\alpha$ for the graph ca-GrQc.

Table~\ref{tab:degree} shows the results for the degree-based setting. They are similar to the ones for the uniform random setting. We are seeing lowest average cost values for Convex \gsemo for cfat200-1 and ca-netscience, and lower average cost values for \gsemo for cfat200-2. The multi-objective results are statistically significant better than the ones for the \ea for these three graphs, but there is no significant difference between the results for \gsemo and Convex \gsemo. For the larger graphs ca-GrQc and Erdos992, the lowest average cost values are obtained by Convex \gsemo in most cases. 8 out of 20 settings show a statistically significant advantage of Convex \gsemo over \ea whereas \gsemo is significantly better than \ea in 7 out of 20 settings. There is no significant difference in the results of Convex \gsemo and \gsemo for these two large graphs.

Finally, we examine the results for the negatively correlated setting given in Table~\ref{tab:ncor}. Here it can be observed that \gsemo performs worse than \ea and Convex \gsemo as the negative correlation induces significant trade-offs in the objective space which may increase the population size of \gsemo drastically. Overall, Convex \gsemo obtains the best results for cfat200-2, cfat200-2, and ca-netscience significantly outperforming \gsemo in all settings and outperforming the \ea in most of them. For ca-GrQc, \ea obtains the best average cost values and statically outperforms both \gsemo and Convex \gsemo. For Erdos992, Convex \gsemo obtains the lowest average cost values. However there is no statistical signficance among the results for this graph.

\begin{table*}
\tiny
    \centering
    \begin{tabular}{|c||c|c|c|c|c|c||c|c|c|c|c|c|} \hline 
       \multirow{3}{*}{Graph} & \multicolumn{6}{|c||}{\bfseries GSEMO} & \multicolumn{6}{|c|}{\bfseries Convex GSEMO}  \\ 
   &    \multicolumn{2}{|c|}{\bfseries uniform random} & \multicolumn{2}{|c|}{\bfseries degree-based} & \multicolumn{2}{|c||}{\bfseries neg. correlated} & \multicolumn{2}{|c|}{\bfseries uniform random} & \multicolumn{2}{|c|}{\bfseries degree-based }  & \multicolumn{2}{|c|}{\bfseries neg. correlated}  \\ 
   & Mean & Std  & Mean &  Std & Mean & Std & Mean &  Std & Mean &  Std & Mean &  Std\\ \hline
  cfat200-1 & 62 & 23 & 3 & 1& 10422 & 562 & 17 & 3& 3 & 1 & 15 & 8 \\
cfat200-2 & 28 & 13 & 4 & 1& 7864 & 161 & 11 & 2& 3 & 1 & 10 & 4 \\
ca-netscience & 51 & 16 & 13 & 3& 1251 & 177 & 16 & 2& 8 & 1 & 19 & 4 \\
ca-GrQc & 8 & 1 & 8 & 2& 29 & 7 & 6 & 1& 6 & 1 & 9 & 1 \\
Erdos992 & 5 & 1 & 4 & 1& 13 & 2 & 4 & 0& 3 & 1 & 7 & 1 \\ \hline
    \end{tabular}
    \caption{Average maximum population size and standard deviation during the 30 runs for \gsemo and Convex \gsemo in the uniform random, degree-based, and negatively correlated setting.}
    \label{tab:popsize}
\end{table*}
We report in Table~\ref{tab:popsize} on the average maximum population size and the standard deviation of the maximum population of \gsemo and Convex \gsemo among the graphs in the different stochastic settings. It can be observed that the population size of \gsemo is significantly higher than the one of Convex \gsemo for cfat200-1, cfat200-2, and ca-netscience. \gsemo obtains a very large mean population size of 10422 for cfat200-1, 7864 for cfat200-2, and 1251 for ca-netscience whereas the mean maximum population size of Convex \gsemo is $15$, $10$, and $19$ for these settings. This clearly shows the reduction in population size when using the convex approach and also explains the significantly better results for Convex \gsemo for these three graphs in Table ~\ref{tab:ncor}.  For the graphs ca-GrQc and Erdos992, the population sizes are comparably small. One explanation might be that the 10M iterations in combination with standard bit mutations and no problem specific operators are not enough to come close enough to the true Pareto front of these problems and that strict improvements for both objectives during the run of the algorithms frequently reduce the population size such that not such large number of trade-offs are produced as in the case of \gsemo for cfat200-1, cfat200-2, and ca-netscience.

\section{Conclusions}
With this paper, we provided the first analysis of evolutionary algorithms for chance constrained combinatorial optimization problems with normally distributed variables where the objective function is a linear combination of such variables, subject to a cardinality constraint or spanning tree constraint. 
For the case of uniform constraints we have shown that there are simple instances where the \ea has an exponential optimization time. Based on these insights we provided a multi-objective formulation which allows Pareto optimization approaches to compute a set of solutions containing for every possible confidence of $\alpha$ an optimal solution based on the classical \gsemo algorithm. As the population size of \gsemo might become exponential in the problem size,  we also provided improved convex evolutionary multi-objective approaches that allow to polynomially bound the population size while still achieving the same theoretical guarantees. We also showed that the results can be transferred to the chanced constrained minimum spanning tree problem.
Finally, we pointed out the effectiveness of our multi-objective approaches by investigating the chance constrained setting of the minimum weight dominating set problem. Our results show that the multi-objective approaches outperform the single objective results obtained by \ea in the examined settings. Furthermore, our investigations for the negatively correlated setting clearly show the benefit of using Convex \gsemo instead of the standard \gsemo approach.

\section*{Acknowledgments}

This work has been supported by the Australian Research Council (ARC) through grant FT200100536 
and by the Independent Research Fund Denmark 
through grant DFF-FNU  8021-00260B.

\begin{table*}[h]
\tiny
    \centering
    \begin{tabular}{|c|c||c|c||c|c|c||c|c|c||c|} \hline 
       \multirow{2}{*}{Graph} & \multirow{2}{*}{$\beta$} &  \multicolumn{2}{|c||}{\bfseries (1+1)EA } & \multicolumn{3}{|c||}{\bfseries GSEMO} & \multicolumn{4}{|c|}{\bfseries Convex GSEMO}  \\ 
  & &     Mean & Std & Mean & Std & $p_1$ & Mean & Std & $p_2$ & $p_3$  \\ \hline
\multirow{10}{*}{cfat200-1} & 0.2 & 3699 & 138 & 3627 & 93 & \textbf{0.046} & \textbf{3622} & 75 & \textbf{0.026} & 0.953\\
& 0.1 & 4076 & 121 & 4003 & 100 & \textbf{0.017} & \textbf{4000} & 79 & \textbf{0.011} & 0.935\\
& 0.01 & 4963 & 149 & \textbf{4885} & 119 & 0.059 & \textbf{4885} & 94 & 0.056 & 0.882\\
& 1.0E-4 & 6082 & 170 & \textbf{6038} & 139 & 0.408 & 6041 & 113 & 0.751 & 0.871\\
& 1.0E-6 & 7004 & 185 & \textbf{6879} & 149 & \textbf{0.012} & 6884 & 129 & \textbf{0.005} & 0.859\\
& 1.0E-8 & 7700 & 201 & \textbf{7571} & 159 & \textbf{0.009} & 7578 & 143 & \textbf{0.012} & 0.802\\
& 1.0E-10 & 8285 & 226 & \textbf{8171} & 166 & \textbf{0.049} & 8178 & 153 & 0.064 & 0.745\\
& 1.0E-12 & 8840 & 257 & \textbf{8708} & 173 & \textbf{0.044} & 8716 & 161 & \textbf{0.041} & 0.790\\
& 1.0E-14 & 9320 & 205 & \textbf{9198} & 179 & \textbf{0.031} & 9206 & 167 & \textbf{0.019} & 0.802\\
& 1.0E-16 & 9818 & 208 & \textbf{9641} & 184 & \textbf{0.002} & 9649 & 172 & \textbf{0.003} & 0.824\\ \hline

\multirow{10}{*}{cfat200-2} & 0.2 & 1876 & 104 & 1781 & 36 & \textbf{0.000} & \textbf{1780} & 29 & \textbf{0.000} & 0.900\\
& 0.1 & 2097 & 104 & 2031 & 40 & \textbf{0.016} & \textbf{2029} & 33 & \textbf{0.014} & 0.912\\
& 0.01 & 2707 & 140 & 2612 & 54 & \textbf{0.012} & \textbf{2609} & 48 & \textbf{0.009} & 0.982\\
& 1.0E-4 & 3468 & 160 & 3366 & 71 & \textbf{0.024} & \textbf{3362} & 67 & \textbf{0.018} & 0.865\\
& 1.0E-6 & 3996 & 217 & 3920 & 85 & 0.918 & \textbf{3911} & 76 & 0.779 & 0.652\\
& 1.0E-8 & 4515 & 208 & 4377 & 94 & \textbf{0.015} & \textbf{4362} & 81 & \textbf{0.008} & 0.589\\
& 1.0E-10 & 4894 & 206 & 4773 & 101 & \textbf{0.047} & \textbf{4755} & 86 & \textbf{0.014} & 0.464\\
& 1.0E-12 & 5194 & 216 & 5128 & 107 & 0.540 & \textbf{5108} & 90 & 0.228 & 0.420\\
& 1.0E-14 & 5561 & 253 & 5453 & 113 & 0.268 & \textbf{5429} & 93 & 0.095 & 0.387\\
& 1.0E-16 & 5854 & 269 & 5746 & 118 & 0.333 & \textbf{5721} & 96 & 0.171 & 0.379\\ \hline

\multirow{10}{*}{ca-netscience} & 0.2 & 33937 & 1485 & 32729 & 970 & \textbf{0.002} & \textbf{32694} & 1044 & \textbf{0.001} & 0.584\\
& 0.1 & 35679 & 1377 & 34252 & 975 & \textbf{0.000} & \textbf{34221} & 1057 & \textbf{0.000} & 0.589\\
& 0.01 & 39295 & 1526 & 37865 & 989 & \textbf{0.000} & \textbf{37843} & 1090 & \textbf{0.000} & 0.652\\
& 1.0E-4 & 44581 & 1853 & 42676 & 1008 & \textbf{0.000} & \textbf{42660} & 1136 & \textbf{0.000} & 0.657\\
& 1.0E-6 & 48021 & 1682 & 46242 & 1024 & \textbf{0.000} & \textbf{46233} & 1173 & \textbf{0.000} & 0.647\\
& 1.0E-8 & 51144 & 1903 & 49199 & 1039 & \textbf{0.000} & \textbf{49194} & 1204 & \textbf{0.000} & 0.734\\
& 1.0E-10 & 53543 & 1598 & \textbf{51776} & 1054 & \textbf{0.000} & \textbf{51776} & 1231 & \textbf{0.000} & 0.723\\
& 1.0E-12 & 56451 & 2372 & \textbf{54089} & 1068 & \textbf{0.000} & 54092 & 1257 & \textbf{0.000} & 0.751\\
& 1.0E-14 & 58337 & 1876 & \textbf{56204} & 1081 & \textbf{0.000} & 56210 & 1281 & \textbf{0.000} & 0.796\\
& 1.0E-16 & 60503 & 2056 & \textbf{58120} & 1094 & \textbf{0.000} & 58129 & 1303 & \textbf{0.000} & 0.830\\ \hline

\multirow{10}{*}{ca-GrQc} & 0.2 & 5543667 & 71362 & 5521866 & 76663 & 0.408 & \textbf{5493337} & 67046 & \textbf{0.009} & 0.086\\
& 0.1 & 5577721 & 79310 & 5587814 & 77062 & 0.690 & \textbf{5559126} & 67282 & 0.255 & 0.084\\
& 0.01 & 5735476 & 103924 & 5744434 & 78015 & 0.425 & \textbf{5715369} & 67843 & 0.615 & 0.084\\
& 1.0E-4 & 5963501 & 72294 & 5953199 & 79296 & 0.929 & \textbf{5923633} & 68598 & \textbf{0.074} & 0.079\\
& 1.0E-6 & 6127198 & 79306 & 6108259 & 80253 & 0.399 & \textbf{6078319} & 69174 & \textbf{0.019} & 0.076\\
& 1.0E-8 & 6245158 & 115796 & 6236959 & 81047 & 0.824 & \textbf{6206709} & 69657 & 0.169 & 0.074\\
& 1.0E-10 & 6348120 & 82176 & 6349284 & 81743 & 0.679 & \textbf{6318765} & 70081 & 0.095 & 0.071\\
& 1.0E-12 & 6508955 & 112190 & 6450188 & 82370 & 0.046 & \textbf{6419426} & 70463 & \textbf{0.001} & 0.067\\
& 1.0E-14 & 6559776 & 97731 & 6542561 & 82945 & 0.734 & \textbf{6511579} & 70814 & 0.069 & 0.065\\
& 1.0E-16 & 6670967 & 104511 & 6626325 & 83468 & 0.143 & \textbf{6595141} & 71133 & \textbf{0.003} & 0.067\\ \hline

\multirow{10}{*}{Erdos992} & 0.2 & 13750812 & 117719 & \textbf{13725406} & 83102 & 0.605 & 13728034 & 89372 & 0.574 & 0.976\\
& 0.1 & \textbf{13835782} & 84823 & 13851522 & 83182 & 0.416 & 13854145 & 89534 & 0.383 & 0.965\\
& 0.01 & 14153639 & 107878 & \textbf{14151038} & 83382 & 0.723 & 14153647 & 89928 & 0.712 & 0.953\\
& 1.0E-4 & 14570235 & 101171 & \textbf{14550280} & 83668 & 0.344 & 14552870 & 90474 & 0.375 & 1.000\\
& 1.0E-6 & \textbf{14840217} & 105886 & 14846818 & 83896 & 0.469 & 14849394 & 90895 & 0.451 & 0.988\\
& 1.0E-8 & 15106991 & 108070 & \textbf{15092949} & 84095 & 0.906 & 15095514 & 91254 & 0.965 & 1.000\\
& 1.0E-10 & 15350942 & 112637 & \textbf{15307765} & 84276 & 0.198 & 15310319 & 91574 & 0.243 & 0.988\\
& 1.0E-12 & 15525936 & 77168 & \textbf{15500739} & 84444 & 0.375 & 15503283 & 91868 & 0.329 & 0.976\\
& 1.0E-14 & 15699173 & 125671 & \textbf{15677400} & 84602 & 0.460 & 15679936 & 92141 & 0.515 & 0.988\\
& 1.0E-16 & 15839330 & 111483 & \textbf{15837595} & 84750 & 0.859 & 15840122 & 92392 & 0.929 & 0.976\\ \hline

    \end{tabular}
    \caption{
    Results for stochastic minimum weight dominating set in the \emph{uniform random setting} with different confidence levels of $\alpha$ where $\alpha=1-\beta$. 
    }
    \label{tab:uniform}
\end{table*}

\begin{table*}[h]
\tiny
    \centering
    \begin{tabular}{|c|c||c|c||c|c|c||c|c|c||c|} \hline 
       \multirow{2}{*}{Graph} & \multirow{2}{*}{$\beta$} &  \multicolumn{2}{|c||}{\bfseries (1+1)EA } & \multicolumn{3}{|c||}{\bfseries GSEMO} & \multicolumn{4}{|c|}{\bfseries Convex GSEMO}  \\ 
  & &     Mean & Std & Mean & Std & $p_1$ & Mean & Std & $p_2$ & $p_3$  \\ \hline
\multirow{10}{*}{cfat200-1} & 0.2 & 4539 & 175 & 4510 & 147 & 0.717 & \textbf{4481} & 136 & 0.918 & 0.813\\
& 0.1 & 4875 & 178 & 4851 & 151 & 0.469 & \textbf{4822} & 140 & 0.297 & 0.734\\
& 0.01 & 5667 & 194 & 5662 & 162 & 0.540 & \textbf{5632} & 151 & 0.549 & 0.836\\
& 1.0E-4 & \textbf{6705} & 175 & 6742 & 176 & 0.201 & 6712 & 165 & 0.204 & 0.836\\
& 1.0E-6 & 7534 & 200 & 7544 & 186 & 0.610 & \textbf{7514} & 177 & 0.636 & 0.836\\
& 1.0E-8 & 8253 & 265 & 8210 & 195 & 0.882 & \textbf{8179} & 187 & 0.779 & 0.824\\
& 1.0E-10 & 8830 & 226 & 8791 & 203 & 0.600 & \textbf{8760} & 195 & 0.535 & 0.824\\
& 1.0E-12 & 9325 & 230 & 9313 & 211 & 0.877 & \textbf{9282} & 203 & 0.900 & 0.824\\
& 1.0E-14 & 9813 & 272 & 9791 & 217 & 0.695 & \textbf{9760} & 210 & 0.684 & 0.824\\
& 1.0E-16 & 10264 & 310 & 10224 & 223 & 0.965 & \textbf{10193} & 217 & 0.929 & 0.824\\ \hline

\multirow{10}{*}{cfat200-2} & 0.2 & 3129 & 221 & \textbf{2963} & 6 & \textbf{0.009} & 2993 & 114 & 0.053 & 0.371\\
& 0.1 & 3248 & 160 & \textbf{3185} & 8 & 0.098 & 3216 & 118 & 0.425 & 0.371\\
& 0.01 & 3778 & 173 & \textbf{3712} & 15 & 0.318 & 3745 & 127 & 0.882 & 0.371\\
& 1.0E-4 & 4653 & 278 & \textbf{4415} & 24 & \textbf{0.002} & 4450 & 140 & \textbf{0.011} & 0.371\\
& 1.0E-6 & 5192 & 277 & \textbf{4937} & 31 & \textbf{0.000} & 4974 & 149 & \textbf{0.000} & 0.371\\
& 1.0E-8 & 5787 & 285 & \textbf{5370} & 37 & \textbf{0.000} & 5409 & 157 & \textbf{0.000} & 0.371\\
& 1.0E-10 & 6006 & 317 & \textbf{5748} & 42 & \textbf{0.000} & 5788 & 164 & \textbf{0.002} & 0.371\\
& 1.0E-12 & 6349 & 316 & \textbf{6088} & 46 & \textbf{0.000} & 6129 & 171 & \textbf{0.003} & 0.371\\
& 1.0E-14 & 6759 & 349 & \textbf{6399} & 50 & \textbf{0.000} & 6441 & 177 & \textbf{0.000} & 0.371\\
& 1.0E-16 & 6956 & 335 & \textbf{6680} & 54 & \textbf{0.000} & 6724 & 182 & \textbf{0.000} & 0.371\\ \hline

\multirow{10}{*}{ca-netscience} & 0.2 & 31530 & 1463 & 27611 & 995 & \textbf{0.000} & \textbf{27587} & 952 & \textbf{0.000} & 0.918\\
& 0.1 & 33723 & 1403 & 29133 & 1022 & \textbf{0.000} & \textbf{29110} & 981 & \textbf{0.000} & 0.929\\
& 0.01 & 36963 & 2082 & 32742 & 1085 & \textbf{0.000} & \textbf{32720} & 1049 & \textbf{0.000} & 0.953\\
& 1.0E-4 & 42136 & 1659 & 37549 & 1170 & \textbf{0.000} & \textbf{37527} & 1139 & \textbf{0.000} & 1.000\\
& 1.0E-6 & 44111 & 2145 & 41114 & 1233 & \textbf{0.000} & \textbf{41092} & 1205 & \textbf{0.000} & 0.941\\
& 1.0E-8 & 47195 & 1626 & 44072 & 1286 & \textbf{0.000} & \textbf{44050} & 1261 & \textbf{0.000} & 0.965\\
& 1.0E-10 & 49034 & 1669 & 46653 & 1331 & \textbf{0.000} & \textbf{46632} & 1309 & \textbf{0.000} & 0.953\\
& 1.0E-12 & 51861 & 2184 & 48973 & 1372 & \textbf{0.000} & \textbf{48951} & 1352 & \textbf{0.000} & 0.953\\
& 1.0E-14 & 53979 & 1816 & 51096 & 1410 & \textbf{0.000} & \textbf{51074} & 1391 & \textbf{0.000} & 0.953\\
& 1.0E-16 & 55646 & 1864 & 53021 & 1444 & \textbf{0.000} & \textbf{52999} & 1427 & \textbf{0.000} & 0.953\\ \hline

\multirow{10}{*}{ca-GrQc} & 0.2 & 4197338 & 67941 & 3998433 & 63560 & \textbf{0.000} & \textbf{3977975} & 56787 & \textbf{0.000} & 0.280\\
& 0.1 & 4233493 & 75658 & 4065409 & 64072 & \textbf{0.000} & \textbf{4044756} & 57294 & \textbf{0.000} & 0.280\\
& 0.01 & 4311732 & 78692 & 4224464 & 65290 & \textbf{0.000} & \textbf{4203348} & 58497 & \textbf{0.000} & 0.294\\
& 1.0E-4 & 4485697 & 61376 & 4436475 & 66917 & \textbf{0.011} & \textbf{4414741} & 60103 & \textbf{0.000} & 0.294\\
& 1.0E-6 & 4628802 & 96005 & 4593946 & 68128 & 0.249 & \textbf{4571753} & 61297 & \textbf{0.020} & 0.301\\
& 1.0E-8 & 4737725 & 63258 & 4724649 & 69135 & 0.460 & \textbf{4702076} & 62289 & \textbf{0.043} & 0.301\\
& 1.0E-10 & 4837627 & 58423 & 4838723 & 70016 & 0.953 & \textbf{4815817} & 63155 & 0.114 & 0.301\\
& 1.0E-12 & 4946776 & 63573 & 4941198 & 70808 & 0.701 & \textbf{4917993} & 63935 & 0.128 & 0.301\\
& 1.0E-14 & \textbf{4996542} & 62242 & 5035010 & 71533 & \textbf{0.040} & 5011532 & 64648 & 0.337 & 0.301\\
& 1.0E-16 & \textbf{5087514} & 71264 & 5120078 & 72192 & 0.069 & 5096352 & 65296 & 0.636 & 0.301\\ \hline

\multirow{10}{*}{Erdos992 ND} & 0.2 & 9352047 & 57702 & 9296596 & 52712 & \textbf{0.001} & \textbf{9296402} & 73513 & \textbf{0.001} & 0.595\\
& 0.1 & 9468588 & 57137 & 9422734 & 53069 & \textbf{0.005} & \textbf{9422553} & 74025 & \textbf{0.004} & 0.584\\
& 0.01 & 9743994 & 67921 & 9722300 & 53917 & 0.294 & \textbf{9722150} & 75244 & 0.124 & 0.564\\
& 1.0E-4 & 10133440 & 74481 & 10121606 & 55047 & 0.813 & \textbf{10121496} & 76869 & 0.383 & 0.564\\
& 1.0E-6 & 10440708 & 77221 & 10418190 & 55886 & 0.329 & 10418111 & 78076 & 0.174 & 0.564\\
& 1.0E-8 & 10657346 & 50014 & 10664361 & 56583 & 0.425 & \textbf{10664307} & 79078 & 0.882 & 0.564\\
& 1.0E-10 & \textbf{10866919} & 56847 & 10879211 & 57191 & 0.391 & 10879179 & 79953 & 0.894 & 0.574\\
& 1.0E-12 & 11090265 & 66762 & 11072215 & 57738 & 0.391 & \textbf{11072202} & 80739 & 0.139 & 0.574\\
& 1.0E-14 & 11264516 & 61580 & \textbf{11248904} & 58239 & 0.301 & 11248909 & 81458 & 0.179 & 0.574\\
& 1.0E-16 & 11414410 & 66047 & \textbf{11409124} & 58693 & 0.906 & 11409146 & 82111 & 0.525 & 0.574\\ \hline

    \end{tabular}
    \caption{
    Results for stochastic minimum weight dominating set in the \emph{degree-based setting} with different confidence levels of $\alpha$ where $\alpha=1-\beta$.
    }
    \label{tab:degree}
\end{table*}

\begin{table*}[h]
\tiny
    \centering
    \begin{tabular}{|c|c||c|c||c|c|c||c|c|c||c|} \hline 
       \multirow{2}{*}{Graph} & \multirow{2}{*}{$\beta$} &  \multicolumn{2}{|c||}{\bfseries (1+1)EA } & \multicolumn{3}{|c||}{\bfseries GSEMO} & \multicolumn{4}{|c|}{\bfseries Convex GSEMO}  \\ 
  & &     Mean & Std & Mean & Std & $p_1$ & Mean & Std & $p_2$ & $p_3$  \\ \hline
\multirow{10}{*}{cfat200-1} &  0.2 & 170311 & 11016 & 227664 & 21698 & \textbf{0.000} & \textbf{169954} & 11955 & 0.762 & \textbf{0.000}\\
& 0.1 & 233716 & 11997 & 284043 & 19488 & \textbf{0.000} & \textbf{231067} & 10945 & 0.501 & \textbf{0.000}\\
& 0.01 & 378240 & 11565 & 417272 & 14737 & \textbf{0.000} & \textbf{375031} & 8780 & 0.363 & \textbf{0.000}\\
& 1.0E-4 & 572211 & 10556 & 592698 & 9229 & \textbf{0.000} & \textbf{564727} & 6283 & 0.004 & \textbf{0.000}\\
& 1.0E-6 & 715342 & 12480 & 719542 & 5799 & 0.220 & \textbf{687779} & 15319 & \textbf{0.000} & \textbf{0.000}\\
& 1.0E-8 & 817353 & 42170 & 786609 & 15154 & \textbf{0.000} & \textbf{729576} & 22084 & \textbf{0.000} & \textbf{0.000}\\
& 1.0E-10 & 798071 & 56201 & 838022 & 20127 & \textbf{0.000} & \textbf{764112} & 25902 & \textbf{0.008} & \textbf{0.000}\\
& 1.0E-12 & 820664 & 30778 & 883991 & 24506 & \textbf{0.000} & \textbf{794903} & 29220 & \textbf{0.004} & \textbf{0.000}\\
& 1.0E-14 & 832737 & 35171 & 926026 & 28527 & \textbf{0.000} & \textbf{822829} & 32145 & 0.301 & \textbf{0.000}\\
& 1.0E-16 & 867904 & 44324 & 963998 & 32057 & \textbf{0.000} & \textbf{847889} & 34612 & 0.080 & \textbf{0.000} \\ \hline

\multirow{10}{*}{cfat200-2} & 0.2 & \textbf{95320} & 4654 & 105428 & 7240 & \textbf{0.000} & 96016 & 4224 & 0.478 & \textbf{0.000}\\
& 0.1 & 140361 & 5769 & 146017 & 6314 & \textbf{0.001} & \textbf{138917} & 4992 & 0.284 & \textbf{0.000}\\
& 0.01 & 244930 & 9460 & 242200 & 4308 & 0.359 & \textbf{240388} & 7884 & 0.135 & \textbf{0.032}\\
& 1.0E-4 & 374151 & 15537 & 335380 & 9379 & \textbf{0.000} & \textbf{326490} & 17856 & \textbf{0.000} & \textbf{0.023}\\
& 1.0E-6 & 356025 & 21961 & 369648 & 13944 & \textbf{0.006} & \textbf{351823} & 19688 & 0.544 & \textbf{0.001}\\
& 1.0E-8 & 380502 & 22350 & 397615 & 17575 & \textbf{0.002} & \textbf{372325} & 21313 & 0.114 & \textbf{0.000}\\
& 1.0E-10 & 392917 & 22408 & 421882 & 20802 & \textbf{0.000} & \textbf{389702} & 22932 & 0.690 & \textbf{0.000}\\
& 1.0E-12 & 415384 & 22260 & 443681 & 23805 & \textbf{0.000} & \textbf{405138} & 24481 & 0.089 & \textbf{0.000}\\
& 1.0E-14 & 430166 & 31161 & 463638 & 26611 & \textbf{0.000} & \textbf{419137} & 26120 & 0.198 & \textbf{0.000}\\
& 1.0E-16 & 441968 & 28484 & 481734 & 29189 & \textbf{0.000} & \textbf{431716} & 27779 & 0.165 & \textbf{0.000}\\ \hline

\multirow{10}{*}{ca-netscience} & 0.2 & \textbf{3843100} & 321904 & 4111813 & 253126 & \textbf{0.001} & 3852578 & 267143 & 0.894 & \textbf{0.001}\\
& 0.1 & \textbf{4219712} & 302664 & 4468346 & 241796 & \textbf{0.002} & 4227101 & 256295 & 0.929 & \textbf{0.001}\\
& 0.01 & 5255877 & 311693 & 5314300 & 215783 & 0.433 & \textbf{5109816} & 232377 & 0.069 & \textbf{0.002}\\
& 1.0E-4 & 6489076 & 306911 & 6439660 & 183709 & 0.734 & \textbf{6269480} & 196614 & \textbf{0.005} & \textbf{0.003}\\
& 1.0E-6 & 7367139 & 294759 & 7271660 & 161498 & 0.139 & \textbf{7123384} & 171641 & \textbf{0.000} & \textbf{0.004}\\
& 1.0E-8 & 8149664 & 318271 & 7958983 & 141356 & \textbf{0.016} & \textbf{7830604} & 152165 & \textbf{0.000} & \textbf{0.005}\\
& 1.0E-10 & 8796292 & 275920 & 8556092 & 121928 & \textbf{0.000} & \textbf{8447513} & 135846 & \textbf{0.000} & \textbf{0.006}\\
& 1.0E-12 & 9396379 & 268369 & 9091751 & 104603 & \textbf{0.000} & \textbf{9001694} & 122148 & \textbf{0.000} & \textbf{0.008}\\
& 1.0E-14 & 9976483 & 292924 & 9581852 & 89388 & \textbf{0.000} & \textbf{9509029} & 110705 & \textbf{0.000} & \textbf{0.010}\\
& 1.0E-16 & 10514393 & 301209 & 10025468 & 74748 & \textbf{0.000} & \textbf{9969076} & 101545 & \textbf{0.000} & \textbf{0.014}\\ \hline

\multirow{10}{*}{ca-GrQc} & 0.2 & \textbf{6945134726} & 153247522 & 7993444910 & 181767271 & \textbf{0.000} & 7885729518 & 175981512 & \textbf{0.000} & \textbf{0.036}\\
& 0.1 & \textbf{7098408979} & 177345894 & 8155414980 & 181110239 & \textbf{0.000} & 8049761823 & 175050598 & \textbf{0.000} & \textbf{0.036}\\
& 0.01 & \textbf{7505293118} & 146259716 & 8540068565 & 179656004 & \textbf{0.000} & 8439324294 & 173046258 & \textbf{0.000} & \textbf{0.049}\\
& 1.0E-4 & \textbf{8123837239} & 156945064 & 9052790645 & 177970022 & \textbf{0.000} & 8958594338 & 170840565 & \textbf{0.000} & \textbf{0.049}\\
& 1.0E-6 & \textbf{8517842407} & 222244050 & 9433595465 & 176918246 & \textbf{0.000} & 9344141399 & 169750950 & \textbf{0.000} & 0.060\\
& 1.0E-8 & \textbf{8865439822} & 135694855 & 9749472103 & 176303455 & \textbf{0.000} & 9664094225 & 169158872 & \textbf{0.000} & 0.079\\
& 1.0E-10 & \textbf{9162900901} & 144772187 & 10025110451 & 175845930 & \textbf{0.000} & 9943339250 & 168816725 & \textbf{0.000} & 0.117\\
& 1.0E-12 & \textbf{9463156242} & 203558050 & 10272721608 & 175503787 & \textbf{0.000} & 10193818611 & 168690636 & \textbf{0.000} & 0.132\\
& 1.0E-14 & \textbf{9718762214} & 138997011 & 10499399221 & 175251354 & \textbf{0.000} & 10423023168 & 168660254 & \textbf{0.000} & 0.156\\
& 1.0E-16 & \textbf{9951938979} & 157329646 & 10704941556 & 175078679 & \textbf{0.000} & 10630780035 & 168707908 & \textbf{0.000} & 0.165\\ \hline

\multirow{10}{*}{Erdos992} & 0.2 & \textbf{28124991273} & 422087412 & 28150596195 & 385281517 & 0.515 & 28142708326 & 421439222 & 0.859 & 0.723\\
& 0.1 & 28590748263 & 417135109 & 28593909556 & 382655226 & 0.756 & \textbf{28586466445} & 418896625 & 0.859 & 0.712\\
& 0.01 & 29651325305 & 413433664 & 29646740263 & 376452287 & 0.871 & \textbf{29640353415} & 412882553 & 0.779 & 0.734\\
& 1.0E-4 & 31048070243 & 418000297 & 31050114495 & 368258069 & 0.941 & \textbf{31045139488} & 404921574 & 0.779 & 0.745\\
& 1.0E-6 & 32126892633 & 408691442 & 32092451304 & 362207224 & 0.965 & \textbf{32088547584} & 399051641 & 0.734 & 0.723\\
& 1.0E-8 & 32985593114 & 400844161 & 32957609437 & 357226726 & 0.953 & \textbf{32954594905} & 394208717 & 0.767 & 0.723\\
& 1.0E-10 & 33735518181 & 381758701 & 33712692952 & 352912242 & 0.953 & \textbf{33710454476} & 390004455 & 0.668 & 0.756\\
& 1.0E-12 & 34430190991 & 384215525 & 34390998799 & 349063131 & 0.779 & \textbf{34389457468} & 386246155 & 0.564 & 0.745\\
& 1.0E-14 & 35096156052 & 362370289 & 35011948096 & 345569370 & 0.487 & \textbf{35011063337} & 382821331 & 0.308 & 0.734\\
& 1.0E-16 & 35667814756 & 401846524 & 35575007063 & 342426114 & 0.280 & \textbf{35574730689} & 379729121 & 0.220 & 0.723\\ \hline

    \end{tabular}
    \caption{
    Results for stochastic minimum weight dominating set in the \emph{negatively correlated setting} with different confidence levels of $\alpha$ where $\alpha=1-\beta$. 
    }
    \label{tab:ncor}
\end{table*}


\begin{thebibliography}{}

\bibitem[Assimi et~al., 2020]{DBLP:conf/ecai/AssimiHXN020}
Assimi, H., Harper, O., Xie, Y., Neumann, A., and Neumann, F. (2020).
\newblock Evolutionary bi-objective optimization for the dynamic
  chance-constrained knapsack problem based on tail bound objectives.
\newblock In {\em {ECAI}}, volume 325, pages 307--314. {IOS} Press.

\bibitem[Ben-Tal et~al., 2009]{BEN:09}
Ben-Tal, A., El~Ghaoui, L., and Nemirovski, A. (2009).
\newblock {\em Robust Optimization}.
\newblock Princeton Series in Applied Mathematics. Princeton University Press.

\bibitem[Charnes and Cooper, 1959]{Charnes}
Charnes, A. and Cooper, W.~W. (1959).
\newblock Chance-constrained programming.
\newblock {\em Management science}, 6(1):73--79.

\bibitem[Doerr, 2020]{DoerrProbabilisticTools}
Doerr, B. (2020).
\newblock Probabilistic tools for the analysis of randomized optimization
  heuristics.
\newblock In Doerr, B. and Neumann, F., editors, {\em Theory of Evolutionary
  Computation~--~Recent Developments in Discrete Optimization}, pages 1--87.
  Springer.

\bibitem[Doerr et~al., 2020]{DBLP:conf/aaai/DoerrD0NS20}
Doerr, B., Doerr, C., Neumann, A., Neumann, F., and Sutton, A.~M. (2020).
\newblock Optimization of chance-constrained submodular functions.
\newblock In {\em {AAAI}}, pages 1460--1467. {AAAI} Press.

\bibitem[Doerr et~al., 2012]{DoerrJohannsenWinzenALGO12}
Doerr, B., Johannsen, D., and Winzen, C. (2012).
\newblock Multiplicative drift analysis.
\newblock {\em Algorithmica}, 64:673--697.

\bibitem[Doerr and Neumann, 2020]{DoerrN20}
Doerr, B. and Neumann, F., editors (2020).
\newblock {\em Theory of Evolutionary Computation---Recent Developments in
  Discrete Optimization}.
\newblock Springer.

\bibitem[Friedrich et~al., 2010]{DBLP:journals/ec/FriedrichHHNW10}
Friedrich, T., He, J., Hebbinghaus, N., Neumann, F., and Witt, C. (2010).
\newblock Approximating covering problems by randomized search heuristics using
  multi-objective models.
\newblock {\em Evolutionary Computation}, 18(4):617--633.

\bibitem[Friedrich and Neumann, 2015]{DBLP:journals/ec/FriedrichN15}
Friedrich, T. and Neumann, F. (2015).
\newblock Maximizing submodular functions under matroid constraints by
  evolutionary algorithms.
\newblock {\em Evolutionary Computation}, 23(4):543--558.

\bibitem[Giel, 2003]{1299908}
Giel, O. (2003).
\newblock Expected runtimes of a simple multi-objective evolutionary algorithm.
\newblock In {\em {CEC}}, volume~3, pages 1918--1925.

\bibitem[Ishii et~al., 1981]{DBLP:journals/dam/IshiiSNN81}
Ishii, H., Shiode, S., Nishida, T., and Namasuya, Y. (1981).
\newblock Stochastic spanning tree problem.
\newblock {\em Discret. Appl. Math.}, 3(4):263--273.

\bibitem[Jansen, 2013]{Jansen13}
Jansen, T. (2013).
\newblock {\em Analyzing Evolutionary Algorithms -- The Computer Science
  Perspective}.
\newblock Springer.

\bibitem[Kratsch and Neumann, 2013]{DBLP:journals/algorithmica/KratschN13}
Kratsch, S. and Neumann, F. (2013).
\newblock Fixed-parameter evolutionary algorithms and the vertex cover problem.
\newblock {\em Algorithmica}, 65(4):754--771.

\bibitem[Laumanns et~al., 2004]{DBLP:journals/tec/LaumannsTZ04}
Laumanns, M., Thiele, L., and Zitzler, E. (2004).
\newblock Running time analysis of multiobjective evolutionary algorithms on
  pseudo-boolean functions.
\newblock {\em {IEEE} Trans. Evol. Comput.}, 8(2):170--182.

\bibitem[Liu et~al., 2013]{Zhang}
Liu, B., Zhang, Q., Fern{\'{a}}ndez, F.~V., and Gielen, G. G.~E. (2013).
\newblock An efficient evolutionary algorithm for chance-constrained
  bi-objective stochastic optimization.
\newblock {\em {IEEE} Trans. Evolutionary Computation}, 17(6):786--796.

\bibitem[Monfared et~al., 2011]{DBLP:journals/scira/MonfaredM011}
Monfared, M.~D., Mohades, A., and Rezaei, J. (2011).
\newblock Convex hull ranking algorithm for multi-objective evolutionary
  algorithms.
\newblock {\em Sci. Iran.}, 18(6):1435--1442.

\bibitem[Neumann and Neumann, 2020]{DBLP:conf/ppsn/NeumannN20}
Neumann, A. and Neumann, F. (2020).
\newblock Optimising monotone chance-constrained submodular functions using
  evolutionary multi-objective algorithms.
\newblock In {\em {PPSN} {(1)}}, pages 404--417. Springer.

\bibitem[Neumann et~al., 2022]{DBLP:journals/corr/abs-2204-05597}
Neumann, A., Xie, Y., and Neumann, F. (2022).
\newblock Evolutionary algorithms for limiting the effect of uncertainty for
  the knapsack problem with stochastic profits.
\newblock In {\em {PPSN} {(1)}}, volume 13398 of {\em Lecture Notes in Computer
  Science}, pages 294--307. Springer.

\bibitem[Neumann, 2007]{DBLP:journals/eor/Neumann07}
Neumann, F. (2007).
\newblock Expected runtimes of a simple evolutionary algorithm for the
  multi-objective minimum spanning tree problem.
\newblock {\em Eur. J. Oper. Res.}, 181(3):1620--1629.

\bibitem[Neumann and Sutton, 2019]{DBLP:conf/foga/0001S19}
Neumann, F. and Sutton, A.~M. (2019).
\newblock Runtime analysis of the {(1} + 1) evolutionary algorithm for the
  chance-constrained knapsack problem.
\newblock In {\em {FOGA}}, pages 147--153. {ACM}.

\bibitem[Neumann and Wegener, 2006]{DBLP:journals/nc/NeumannW06}
Neumann, F. and Wegener, I. (2006).
\newblock Minimum spanning trees made easier via multi-objective optimization.
\newblock {\em Nat. Comput.}, 5(3):305--319.

\bibitem[Neumann and Witt, 2010]{NeumannW10}
Neumann, F. and Witt, C. (2010).
\newblock {\em Bioinspired Computation in Combinatorial Optimization --
  Algorithms and Their Computational Complexity}.
\newblock Springer.

\bibitem[Neumann and Witt, 2022]{DBLP:conf/ijcai/0001W22}
Neumann, F. and Witt, C. (2022).
\newblock Runtime analysis of single- and multi-objective evolutionary
  algorithms for chance constrained optimization problems with normally
  distributed random variables.
\newblock In {\em {IJCAI}}, pages 4800--4806. ijcai.org.

\bibitem[Poojari and Varghese, 2008]{poojari}
Poojari, C.~A. and Varghese, B. (2008).
\newblock Genetic algorithm based technique for solving chance constrained
  problems.
\newblock {\em Eur. J. Oper. Res.}, 185(3):1128--1154.

\bibitem[Qian et~al., 2017]{DBLP:conf/ijcai/QianSYT17}
Qian, C., Shi, J., Yu, Y., and Tang, K. (2017).
\newblock On subset selection with general cost constraints.
\newblock In {\em {IJCAI}}, pages 2613--2619.

\bibitem[Qian et~al., 2015]{DBLP:conf/nips/QianYZ15}
Qian, C., Yu, Y., and Zhou, Z. (2015).
\newblock Subset selection by {P}areto optimization.
\newblock In {\em {NIPS}}, pages 1774--1782.

\bibitem[Roostapour et~al., 2022]{DBLP:conf/aaai/RoostapourN0019}
Roostapour, V., Neumann, A., Neumann, F., and Friedrich, T. (2022).
\newblock Pareto optimization for subset selection with dynamic cost
  constraints.
\newblock {\em Artif. Intell.}, 302:103597.

\bibitem[Rossi and Ahmed, 2015]{nr}
Rossi, R.~A. and Ahmed, N.~K. (2015).
\newblock The network data repository with interactive graph analytics and
  visualization.
\newblock In {\em {AAAI}}, pages 4292--4293. {AAAI} Press.

\bibitem[Xie et~al., 2019]{DBLP:conf/gecco/XieHAN019}
Xie, Y., Harper, O., Assimi, H., Neumann, A., and Neumann, F. (2019).
\newblock Evolutionary algorithms for the chance-constrained knapsack problem.
\newblock In {\em {GECCO}}, pages 338--346. {ACM}.

\bibitem[Xie et~al., 2020]{DBLP:conf/gecco/XieN020}
Xie, Y., Neumann, A., and Neumann, F. (2020).
\newblock Specific single- and multi-objective evolutionary algorithms for the
  chance-constrained knapsack problem.
\newblock In {\em {GECCO}}, pages 271--279. {ACM}.

\bibitem[Xie et~al., 2021a]{DBLP:conf/gecco/XieN021}
Xie, Y., Neumann, A., and Neumann, F. (2021a).
\newblock Heuristic strategies for solving complex interacting stockpile
  blending problem with chance constraints.
\newblock In {\em {GECCO}}, pages 1079--1087. {ACM}.

\bibitem[Xie et~al., 2021b]{DBLP:conf/gecco/XieN0S21}
Xie, Y., Neumann, A., Neumann, F., and Sutton, A.~M. (2021b).
\newblock Runtime analysis of {RLS} and the {(1+1)} {EA} for the
  chance-constrained knapsack problem with correlated uniform weights.
\newblock In {\em {GECCO}}, pages 1187--1194. {ACM}.

\bibitem[Zhou et~al., 2019]{DBLP:books/sp/ZhouYQ19}
Zhou, Z., Yu, Y., and Qian, C. (2019).
\newblock {\em Evolutionary learning: Advances in theories and algorithms}.
\newblock Springer.

\end{thebibliography}
\end{document}